\pdfoutput=1
\documentclass{article}

\PassOptionsToPackage{numbers, compress, sort}{natbib}
\usepackage[final]{nips_2018}

\usepackage[utf8]{inputenc} %
\usepackage[T1]{fontenc}    %
\usepackage{tikz}
\usepackage{onimage}
\usepackage{hyperref}       %
\hypersetup{
    colorlinks,
    linkcolor={red!50!black},
    citecolor={blue!50!black},
    urlcolor={blue!80!black}
}
\usepackage{url}            %
\usepackage{booktabs}       %
\usepackage{amsfonts}       %
\usepackage{nicefrac}       %
\usepackage{microtype}      %

\usepackage{caption}
\usepackage{amsmath, amsthm}
\usepackage[shortlabels]{enumitem}
\usepackage{graphicx}
\usepackage{enumitem}

\usepackage{xspace}
\newcommand*{\eg}{e.g.\@\xspace}
\newcommand*{\ie}{i.e.\@\xspace}
\newcommand*{\wrt}{w.r.t.\@\xspace}

\newcommand{\E}{\mathbb{E}}

\newcommand{\Ps}{\mathcal{P}(\Omega)}
\DeclareMathOperator{\Tr}{Tr}

\DeclareMathOperator{\IS}{IS} %
\DeclareMathOperator{\FID}{FID} %
\newcommand{\RR}{\mu}
\newcommand{\coolname}{PRD}

\renewcommand{\P}{P}
\newcommand{\Q}{Q}

\renewcommand{\Ps}{\hat{\P}} %
\newcommand{\Qs}{\hat{\Q}} %
\newcommand{\mP}{\nu_P}
\newcommand{\mQ}{\nu_Q}
\newcommand{\supp}{\operatorname{supp}}

\newcommand{\iPQ}{S} %
\newcommand{\iPQc}{\overline{S}} %

\newcommand{\PR}{\operatorname{\coolname{}}(\Q, \P)}
\newcommand{\PRt}{\operatorname{\coolname{}}(\P, \Q)}
\newcommand{\PRa}{\widehat{\operatorname{\coolname{}}}(\Q, \P)}

\newtheorem{theorem}{Theorem}
\newtheorem{lemma}{Lemma}
\newtheorem{definition}{Definition}

\graphicspath{{./figs/}}

\title{Assessing Generative Models via Precision and Recall}
\author{
  Mehdi S. M. Sajjadi\thanks{This work was done during an internship at Google Brain.\newline Correspondence: \href{http://msajjadi.com}{msajjadi.com}, \href{mailto:bachem@google.com}{bachem@google.com}, \href{mailto:lucic@google.com}{lucic@google.com}.} \\
  MPI for Intelligent Systems,\\
  Max Planck ETH Center \\ for Learning Systems \And
  Olivier Bachem \\ Google Brain \And
  Mario Lucic \\ Google Brain \AND
  Olivier Bousquet \\ Google Brain \And
  Sylvain Gelly\\ Google Brain
}
\begin{document}
\maketitle

\begin{abstract}
  Recent advances in generative modeling have led to an increased interest in
  the study of statistical divergences as means of model comparison. Commonly
  used evaluation methods, such as the Fr\'echet Inception Distance (FID), correlate well
  with the perceived quality of samples and are sensitive to mode dropping. However, these metrics are unable to distinguish between different failure cases since they only yield
  one-dimensional scores. We propose a novel definition of precision and recall
  for distributions which disentangles the divergence into two separate dimensions.
  The proposed notion is intuitive, retains desirable properties, and naturally
  leads to an efficient algorithm that can be used to evaluate generative
  models. We relate this notion to total variation as well as to recent evaluation
  metrics such as Inception Score and FID. To demonstrate the practical
  utility of the proposed approach we perform an empirical study on several
  variants of Generative Adversarial Networks and Variational Autoencoders.
  In an extensive set of experiments we show that the proposed metric is able to
  disentangle the quality of generated samples from the coverage of the target distribution.
\end{abstract}

\vspace{-1.0mm} %
\section{Introduction}
\vspace{-1.5mm} %
\label{sec:introduction}

  Deep generative models, such as Variational Autoencoders
  (VAE)~\cite{kingma2013auto} and Generative Adversarial Networks
  (GAN)~\cite{goodfellow2014generative}, have received a great deal of attention
  due to their ability to learn complex, high-dimensional distributions. One of the biggest impediments to future research is the lack of quantitative evaluation methods to accurately assess the
  quality of trained models. Without a proper evaluation metric researchers
  often need to visually inspect generated samples or resort to qualitative
  techniques which can be subjective. One of the
  main difficulties for quantitative assessment lies in the fact that the
  distribution is only specified implicitly -- one can learn to sample from a
  predefined distribution, but cannot evaluate the likelihood efficiently. In fact, even if likelihood computation
  were computationally tractable, it might be inadequate and misleading for high-dimensional problems~\cite{theis2015note}.

  As a result, surrogate metrics are often used to assess the quality of the trained models. Some proposed measures,
  such as Inception Score (IS)~\cite{salimans2016improved} and Fr\'echet
  Inception Distance (FID)~\cite{heusel2017gans}, have shown promising results in
  practice. In particular, FID has been shown to be robust to image corruption, it correlates well with the visual fidelity of the samples, and it can be computed on unlabeled data.

  However, all of the metrics commonly applied to evaluating generative models share a
  crucial weakness: Since they yield a one-dimensional score, they are unable to
  distinguish between different failure cases. For example, the generative models shown in Figure~\ref{fig:winner1} obtain similar FIDs but exhibit different sample characteristics: the model on the left trained on MNIST~\cite{mnist} produces realistic samples, but only generates a subset of the digits. On the other hand, the model on the right produces low-quality samples which appear to cover all digits. A similar effect can be observed on the CelebA~\cite{liu2015faceattributes} data set. In this work we argue that a single-value summary is not adequate to compare generative models.

  Motivated by this shortcoming, we present a novel approach which disentangles the divergence between distributions into two components: \emph{precision} and \emph{recall}. Given a reference distribution $\P$ and a learned distribution $\Q$,
  precision intuitively measures the quality of samples from $\Q$, while recall
  measures the proportion of $\P$ that is covered by $\Q$. Furthermore, we propose an elegant algorithm which can compute these quantities based on samples from $\P$ and $\Q$. In particular, using this approach we are able to quantify the degree of \emph{mode dropping} and \emph{mode inventing} based on samples from the true and the learned distributions.

  \textbf{Our contributions:}
  \textbf{(1)} We introduce a novel definition of precision and recall for
  distributions and prove that the notion is theoretically sound and has
  desirable properties,
  \textbf{(2)} we propose an efficient algorithm to compute these quantities,
  \textbf{(3)} we relate these notions to total variation, IS and FID,
  \textbf{(4)} we demonstrate that in practice one can quantify the degree of mode
  dropping and mode inventing on real world data sets (image and text data), and
  \textbf{(5)} we compare several types of generative models based on the proposed approach -- to our knowledge, this is the first metric that experimentally confirms the folklore that GANs often produce "sharper" images, but can suffer from mode collapse (high precision, low recall), while VAEs produce "blurry" images, but cover more modes of the distribution (low precision, high recall).

\vspace{-1.0mm} %
\section{Background and Related Work}
\vspace{-1.5mm} %
\label{sec:relatedwork}

  \begin{figure}[t]
  \centering
    \begin{center}
      \setlength{\tabcolsep}{2pt}
      \begin{tabular}{cc}
        \begin{tabular}{c}
          \includegraphics[trim={0 112 0 0},clip,width=0.29\linewidth]{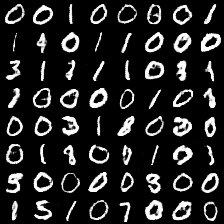} \\
          \includegraphics[trim={0 256 0 0},clip,width=0.29\linewidth]{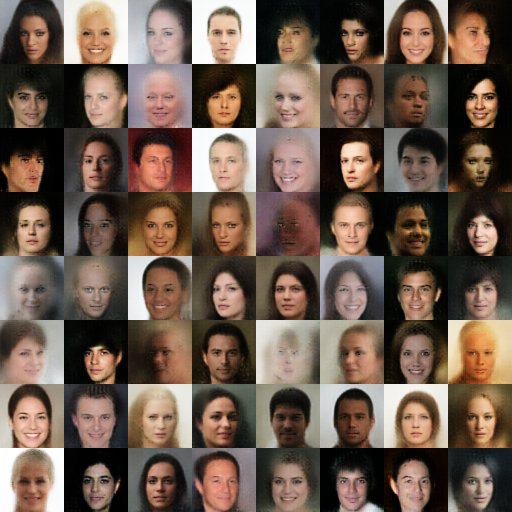}
        \end{tabular}
        \begin{tabular}{c}
          \vspace{-2mm}
          \includegraphics[width=0.35\linewidth]{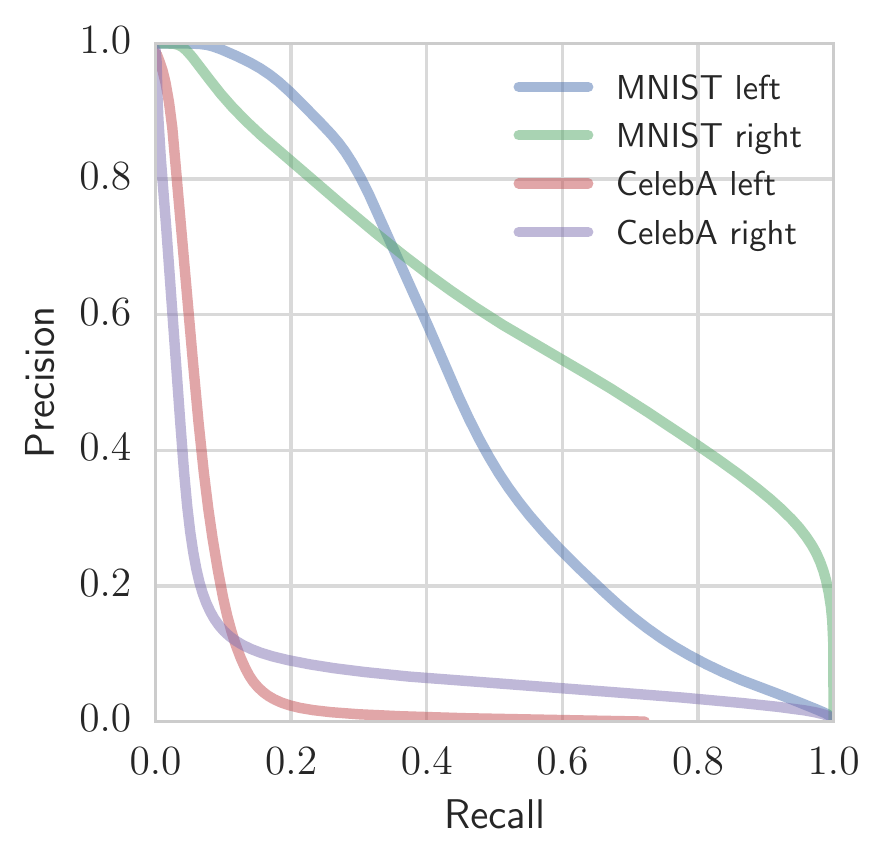}
        \end{tabular}
        \begin{tabular}{c}
          \includegraphics[trim={0 112 0 0},clip,width=0.29\linewidth]{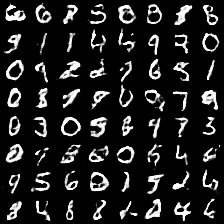} \\
          \includegraphics[trim={0 256 0 0},clip,width=0.29\linewidth]{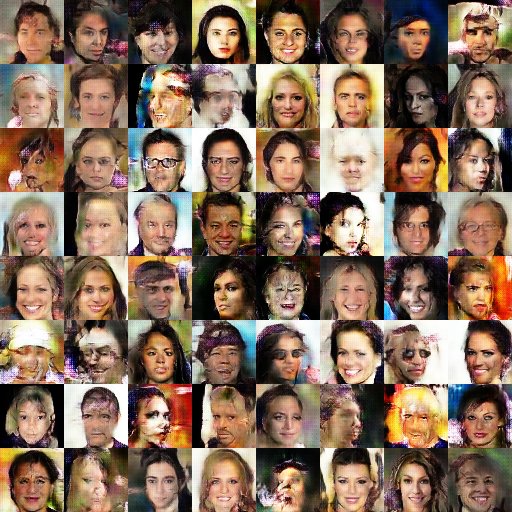}
        \end{tabular}
      \end{tabular}
    \end{center}
    \hfill
    \vspace{-3mm}
    \caption{Comparison of GANs trained on MNIST and CelebA. Although the models
    obtain a similar FID on each data set (32/29 for MNIST and 65/62 for CelebA),
    their samples look very different. For example, the model on the left
    produces reasonably looking faces on CelebA, but too many dark images. In contrast, the model on the right produces more artifacts, but more varied images. By the proposed metric (middle), the models on the left achieve  higher precision and lower recall than the models on the right, which suffices to successfully distinguishing between the
    failure cases.}
    \vspace{-3mm}
    \label{fig:winner1}
  \end{figure}

  The task of evaluating generative models is an active research area. Here we
  focus on recent work in the context of deep generative models for image and text
  data. Classic approaches relying on comparing log-likelihood have received some criticism due the fact that one can achieve high likelihood, but low image quality, and conversely, high-quality images but low likelihood~\cite{theis2015note}. While the likelihood can be approximated in some settings, kernel density estimation in high-dimensional spaces is extremely challenging~\cite{wu2016quantitative, theis2015note}.
  Other failure modes related to density estimation in high-dimensional spaces have been elaborated in~\cite{theis2015note, huszar2015not}. A recent review of popular approaches is presented in~\cite{borji2018pros}.

  The Inception Score (IS)~\cite{salimans2016improved} offers a way to quantitatively evaluate the
  quality of generated samples in the context of image data. Intuitively, the
  conditional label distribution $p(y|x)$ of samples containing meaningful
  objects should have low entropy, while the label distribution over the whole
  data set $p(y)$ should have high entropy. Formally, $\IS(G) = \exp(\E_{x \sim
  G}[d_{KL}(p(y|x), p(y)])$. The score is computed based on a classifier
  (Inception network trained on ImageNet). IS necessitates a labeled
  data set and has been found to be weak at providing guidance for model comparison~\cite{noteoninceptionscore}.

  The FID~\cite{heusel2017gans} provides an alternative approach which requires
  no labeled data. The samples are first embedded in some feature space (\eg, a
  specific layer of Inception network for images). Then, a continuous multivariate
  Gaussian is fit to the data and the distance computed as $\FID(x, g) = ||\mu_x -
  \mu_g||_2^2 + \Tr(\Sigma_x + \Sigma_g - 2(\Sigma_x\Sigma_g)^\frac12)$, where
  $\mu$ and $\Sigma$ denote the mean and covariance of the corresponding
  samples. FID is sensitive to both the addition of spurious modes as well as to
  mode dropping (see Figure~\ref{fig:modes_exp} and results in~\cite{lucic2017gans}). \cite{binkowski2018demystifying} recently introduced an unbiased alternative to FID, the \emph{Kernel Inception Distance}. While unbiased, it shares an extremely high Spearman rank-order correlation with FID~\cite{kurach2018gan}.

  Another approach is to train a classifier between the real and fake
  distributions and to use its accuracy on a test set as a proxy for the quality
  of the samples~\cite{lopez2016revisiting, im2018quantitatively}. This approach necessitates training of a classifier for each model which is seldom practical. Furthermore, the classifier might detect a single dimension
  where the true and generated samples differ (\eg, barely visible artifacts in
  generated images) and enjoy high accuracy, which runs the risk of assigning lower quality to a better model.

  To the best of our knowledge, all commonly used metrics for evaluating
  generative models are one-dimensional in that they only yield a single score
  or distance. A notion of precision and recall has previously been introduced
  in~\cite{lucic2017gans} where the authors compute the distance to the manifold
  of the true data and use it as a proxy for precision and recall on a synthetic data set. Unfortunately, it is not possible to compute this quantity for more
  complex data sets.

  \begin{figure}[t]
    \begin{minipage}[b]{0.6\linewidth}
      \includegraphics[width=\linewidth]{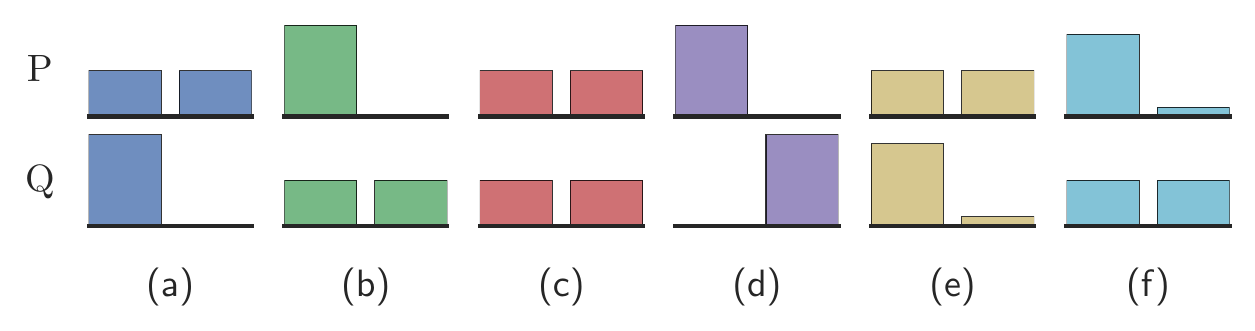}
      \vspace{-6mm}
      \caption{Intuitive examples of $\P$ and $\Q$\label{fig:examples}.}
      \includegraphics[width=\linewidth]{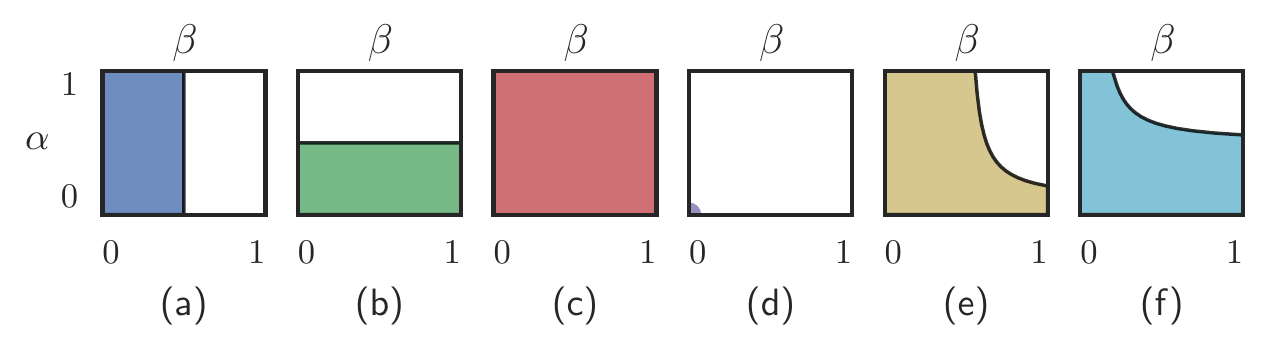}
      \vspace{-6mm}
      \caption{$\PR$ for the examples above\label{fig:examples-pr}.}
    \end{minipage}
    \hfill
    \begin{minipage}[b]{0.38\linewidth}
      \vspace{-2mm}
      \includegraphics[width=\linewidth]{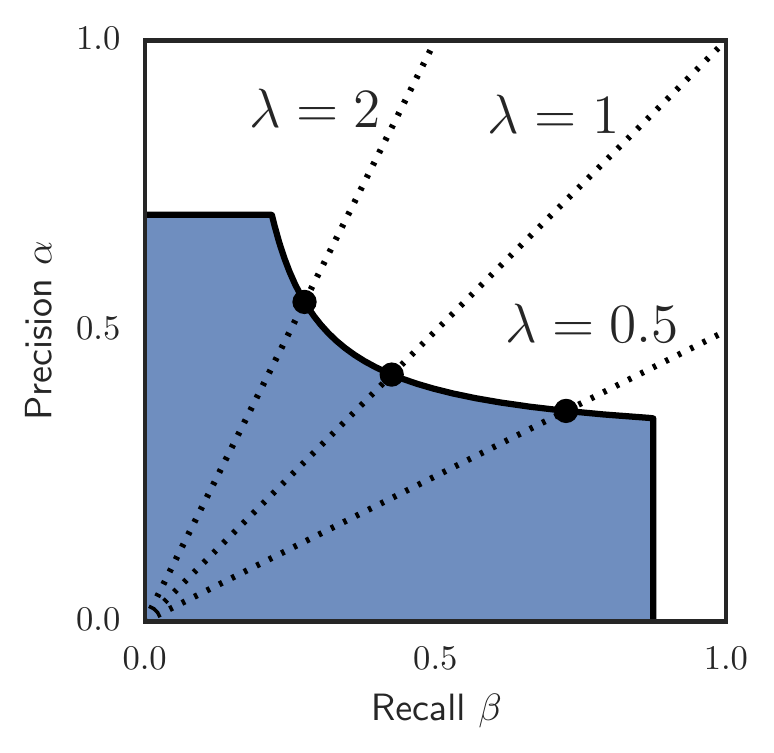}
      \vspace{-6mm}
      \caption{Illustration of the algorithm\label{fig:algorithm}.}
    \end{minipage}
  \end{figure}

\vspace{-1.0mm} %
\section{\coolname{}: Precision and Recall for Distributions}
\vspace{-1.5mm} %
\label{sec:method}

  In this section, we derive a novel notion of precision and recall to compare a
  distribution $\Q$ to a reference distribution $\P$. The key intuition is that
  \emph{precision} should measure how much of $\Q$ can be generated by a
  ``part'' of $\P$ while \emph{recall} should measure how much of $\P$ can be
  generated by a ``part'' of $\Q$. Figure~\ref{fig:examples} (a)-(d) show four
  toy examples for $\P$ and $\Q$ to visualize this idea:
  (a) If $\P$ is bimodal and $\Q$ only captures one of the modes, we should have
  perfect precision but only limited recall.
  (b) In the opposite case, we should have perfect recall but only limited
  precision.
  (c) If $\Q=\P$, we should have perfect precision and recall.
  (d) If the supports of $\P$ and $\Q$ are disjoint, we should have zero
  precision and recall.

\vspace{-1.0mm} %
\subsection{Derivation}
\vspace{-1.5mm} %
\label{sec:derivation}

  Let $\iPQ =\supp(\P) \cap\supp(\Q)$ be the (non-empty) intersection of the
  supports\footnote{For a distribution $\P$ defined
  on a finite state space $\Omega$, we define $\supp(\P)=\lbrace \omega \in
  \Omega \mid \P(\omega)>0\rbrace$.}
  of $\P$ and $\Q$. Then, $\P$ may be viewed as a two-component mixture
  where the first component $\P_\iPQ$ is a probability distribution on $\iPQ$
  and the second component $\P_{\iPQc}$ is defined on the complement of $\iPQ$.
  Similarly, $\Q$ may be rewritten as a mixture of $\Q_\iPQ$ and $\Q_{\iPQc}$.
  More formally, for some $\bar{\alpha}, \bar{\beta} \in(0,1]$, we define
  \begin{equation}
    \label{eqn:firstdecomp}
    \P = \bar{\beta} \P_\iPQ + (1-\bar{\beta}) \P_{\iPQc} \quad \text{and} \quad  \Q = \bar{\alpha} \Q_\iPQ + (1-\bar{\alpha}) \Q_{\iPQc}.
  \end{equation}
  This decomposition allows for a natural interpretation: $\P_{\iPQc}$ is the
  part of $\P$ that cannot be generated by $\Q$, so its mixture weight
  $1-\bar{\beta}$ may be viewed as a loss in recall. Similarly, $\Q_{\iPQc}$ is
  the part of $\Q$ that cannot be generated by $\P$, so $1-\bar{\alpha}$ may be
  regarded as a loss in precision. In the case where $\P_\iPQ=\Q_{\iPQ}$, \ie,
  the distributions $\P$ and $\Q$ agree on $\iPQ$ up to scaling,
  $\bar{\alpha}$ and $\bar{\beta}$ provide us with a simple two-number precision
  and recall summary satisfying the examples in Figure~\ref{fig:examples}
  (a)-(d).

  If $\P_\iPQ\neq\Q_{\iPQ}$, we are faced with a conundrum: Should the differences
  in $\P_\iPQ$ and $\Q_{\iPQ}$ be attributed to losses in precision or recall?
  Is $\Q_{\iPQ}$ inadequately ``covering'' $\P_\iPQ$ or is it generating
  ``unnecessary'' noise? Inspired by PR curves for binary classification, we
  propose to resolve this predicament by providing a trade-off between precision
  and recall instead of a two-number summary for any two distributions $\P$ and
  $\Q$. To parametrize this trade-off, we consider a distribution $\RR$ on
  $\iPQ$ that signifies a ``true'' common component of $\P_\iPQ$ and $\Q_{\iPQ}$
  and similarly to \eqref{eqn:firstdecomp}, we decompose both $\P_\iPQ$ and
  $\Q_{\iPQ}$ as

  \begin{equation}
    \label{eqn:seconddecomp}
    \P_\iPQ = \beta' \RR + (1-\beta') \P_{\RR} \quad \text{and} \quad  \Q_\iPQ = \alpha' \RR + (1-\alpha') \Q_{\RR}.
  \end{equation}

  The distribution $\P_\iPQ$ is viewed as a two-component mixture where the
  first component is $\RR$ and the second component $\P_{\RR}$ signifies the
  part of $\P_\iPQ$ that is ``missed'' by $\Q_{\iPQ}$ and should thus be
  considered a recall loss. Similarly, $\Q_{\iPQ}$ is decomposed into $\RR$ and
  the part $\Q_{\RR}$ that signifies noise and should thus be considered a
  precision loss. As $\RR$ is varied, this leads to a trade-off between
  precision and recall.

  It should be noted that unlike PR curves for binary classification where different thresholds lead to different classifiers, trade-offs between precision and recall here do not constitute different models or distributions -- the proposed \coolname{} curves only serve as a description of the characteristics of the model with respect to the target distribution.

\vspace{-1.0mm} %
\subsection{Formal definition}
\vspace{-1.5mm} %
\label{sec:formaldefinition}

  For simplicity, we consider distributions $\P$ and $\Q$ that are defined on a
  finite state space, though the notion of precision and recall can be extended to arbitrary distributions. By combining \eqref{eqn:firstdecomp}
  and~\eqref{eqn:seconddecomp}, we obtain the following formal definition of
  precision and recall.

  \begin{definition}
  \label{def:precisionrecall}
    For $\alpha, \beta \in (0, 1]$, the probability distribution $\Q$ has
    precision $\alpha$ at recall $\beta$ \wrt $\P$ if there exist distributions
    $\RR$, $\mP$ and $\mQ$ such that
    \begin{equation}
    \label{eqn:prdefinition}
      \P = \beta \RR + (1-\beta) \mP \quad \text{and} \quad  \Q = \alpha \RR + (1-\alpha) \mQ.
    \end{equation}
  \end{definition}

  The component $\mP$ denotes the part of $\P$ that is ``missed'' by $\Q$ and
  encompasses both $\P_{\iPQc}$ in \eqref{eqn:firstdecomp} and $\P_\RR$ in
  \eqref{eqn:seconddecomp}. Similarly, $\mQ$ denotes the noise part of $\Q$ and
  includes both $\Q_{\iPQc}$ in \eqref{eqn:firstdecomp} and $\Q_\RR$ in
  \eqref{eqn:seconddecomp}.

  \begin{definition}
  \label{def:precisionrecallset}
    The set of attainable pairs of precision and recall of a distribution $\Q$
    \wrt a distribution $\P$ is denoted by $\PR$ and it consists of all
    $(\alpha, \beta)$ satisfying Definition~\ref{def:precisionrecall} and the
    pair $(0,0)$.
  \end{definition}

  The set $\PR$ characterizes the above-mentioned trade-off between precision
  and recall and can be visualized similarly to PR curves in binary classification:
  Figure~\ref{fig:examples-pr} (a)-(d) show the set $\PR$ on a 2D-plot for the
  examples (a)-(d) in Figure~\ref{fig:examples}. Note how the plot distinguishes
  between (a) and (b): Any symmetric evaluation method (such as FID) assigns
  these cases the same score although they are highly different. The
  interpretation of the set $\PR$ is further aided by the following set of basic
  properties which we prove in Section~\ref{sec:app:proof:properties} in the
  appendix.

  \begin{theorem}
  \label{thm:properties}
    Let $\P$ and $\Q$ be probability distributions defined on a finite state
    space $\Omega$. The set $\PR$ satisfies the following properties:
    \begin{enumerate}[(i),itemsep=-1mm]
      \item $(1,1)\in\PR \;\;\Leftrightarrow\;\; \Q = \P$
      \hfill(equality)
      \label{prop:equality}
      \item $\PR = \lbrace(0,0)\rbrace \;\;\Leftrightarrow\;\; \supp(\Q) \cap\supp(\P) = \emptyset$
      \hfill(disjoint supports)
      \label{prop:disjoint}
      \item  $\Q(\supp(\P)) = \bar{\alpha}=\max_{(\alpha,\beta)\in\PR}\alpha$
      \hfill (max precision)
      \label{prop:maxprecision}
      \item $\P(\supp(\Q)) = \bar{\beta}=\max_{(\alpha,\beta)\in\PR}\beta$
      \hfill (max recall)
      \label{prop:maxrecall}
      \item $(\alpha',\beta')\in\PR$ \;if\; $\alpha'\in(0,\alpha]$, $\beta'\in(0,\beta]$, $(\alpha,\beta)\in\PR$
      \hfill(monotonicity)
      \label{prop:monotonicity}
      \item $(\alpha,\beta)\in\PR \;\;\Leftrightarrow\;\; (\beta,\alpha)\in \PRt$
      \hfill(duality)
      \label{prop:duality}
    \end{enumerate}
  \end{theorem}

  Property~\ref{prop:equality} in combination with
  Property~\ref{prop:monotonicity} guarantees that $\Q=\P$ if the set $\PR$
  contains the interior of the unit square, see case (c) in
  Figures~\ref{fig:examples} and~\ref{fig:examples-pr}. Similarly,
  Property~\ref{prop:disjoint} assures that whenever there is no overlap between
  $\P$ and $\Q$, $\PR$ only contains the origin, see case (d) of
  Figures~\ref{fig:examples} and~\ref{fig:examples-pr}.
  Properties~\ref{prop:maxprecision} and~\ref{prop:maxrecall} provide a
  connection to the decomposition in~\eqref{eqn:firstdecomp} and allow an
  analysis of the cases (a) and (b) in Figures~\ref{fig:examples}
  and~\ref{fig:examples-pr}: As expected, $\Q$ in (a) achieves a maximum
  precision of 1 but only a maximum recall of 0.5 while in (b), maximum recall
  is 1 but maximum precision is 0.5. Note that the quantities $\bar{\alpha}$ and $\bar{\beta}$ here are by construction the same as in~\eqref{eqn:firstdecomp}. Finally, Property~\ref{prop:duality}
  provides a natural interpretation of precision and recall: The precision of
  $\Q$ \wrt $\P$ is equal to the recall of $\P$ \wrt $\Q$ and \emph{vice versa}.

  Clearly, not all cases are as simple as the examples (a)-(d) in
  Figures~\ref{fig:examples} and~\ref{fig:examples-pr}, in particular if $\P$
  and $\Q$ are different on the intersection $\iPQ$ of their support. The
  examples (e) and (f) in Figure~\ref{fig:examples} and the resulting sets $\PR$
  in Figure~\ref{fig:examples-pr} illustrate the importance of the trade-off
  between precision and recall as well as the utility of the set $\PR$. In both
  cases, $\P$ and $\Q$ have the same support while $\Q$ has high precision and
  low recall in case (e) and low precision and high recall in case (f). This is
  clearly captured by the sets $\PR$. Intuitively, the examples (e) and (f) may
  be viewed as noisy versions of the cases (a) and (b) in
  Figure~\ref{fig:examples}.

\vspace{-1.0mm} %
\subsection{Algorithm}
\vspace{-1.5mm} %
\label{sec:algorithm}

  Computing the set $\PR$ based on Definitions~\ref{def:precisionrecall}
  and~\ref{def:precisionrecallset} is non-trivial as one has to check whether
  there exist suitable distributions $\RR$, $\mP$ and $\mQ$ for all possible
  values of $\alpha$ and $\beta$. We introduce an equivalent definition of $\PR$
  in Theorem~\ref{thm:redefinition} that does not depend on the distributions
  $\RR$, $\mP$ and $\mQ$ and that leads to an elegant algorithm to compute
  practical \coolname{} curves.

  \begin{theorem}[]
  \label{thm:redefinition}
    Let $\P$ and $\Q$ be two probability distributions defined on a finite state
    space $\Omega$. For $\lambda > 0$ define the functions
    \begin{equation}
    \label{eqn:alphabetalambda}
      \alpha(\lambda) = \sum_{\omega\in\Omega}\min\left(\lambda\P(\omega), \Q(\omega)\right)\quad\text{and}\quad \beta(\lambda) = \sum_{\omega\in\Omega}\min\left(\P(\omega), \frac{\Q(\omega)}\lambda\right).
    \end{equation}
    Then, it holds that
    \begin{equation*}
      \PR=\left\lbrace\left(\theta\alpha(\lambda), \theta\beta(\lambda)\right)\mid\lambda\in(0,\infty), \theta\in[0,1]\right\rbrace.
    \end{equation*}
  \end{theorem}

  We prove the theorem in Section~\ref{sec:app:proof:redefinition} in the
  appendix. The key idea of Theorem~\ref{thm:redefinition} is illustrated in
  Figure~\ref{fig:algorithm}: The set of $\PR$ may be viewed as a union of
  segments of the lines $\alpha=\lambda\beta$ over all $\lambda\in(0,\infty)$.
  Each segment starts at the origin $(0,0)$ and ends at the maximal achievable
  value $\left(\alpha(\lambda), \beta(\lambda)\right)$. This provides a
  surprisingly simple algorithm to compute $\PR$ in practice: Simply compute
  pairs of $\alpha(\lambda)$ and $\beta(\lambda)$ as defined in
  \eqref{eqn:alphabetalambda} for an equiangular grid of values of $\lambda$.
  For a given angular resolution $m\in\mathbb{N}$, we compute
  \begin{equation*}
    \PRa = \left\lbrace \left(\alpha(\lambda), \beta(\lambda)\right)  \mid \lambda \in \Lambda\right\rbrace \quad\text{where}\quad \Lambda = \left\lbrace\tan\left(\tfrac{i}{m+1}\tfrac{\pi}{2}\right)\mid i =1, 2,\dots,m\right\rbrace.
  \end{equation*}
  To compare different distributions $\Q_i$, one may simply plot their
  respective \coolname{} curves $\widehat{\operatorname{\coolname{}}}(\Q_i, \P)$, while an approximation of the full sets $\operatorname{\coolname{}}(\Q_i, \P)$ may be computed by
  interpolation between $\widehat{\operatorname{\coolname{}}}(\Q_i, \P)$ and the origin. An implementation of the algorithm is available at \href{https://github.com/msmsajjadi/precision-recall-distributions}{https://github.com/msmsajjadi/precision-recall-distributions}.

\vspace{-1.0mm} %
\subsection{Connection to total variation distance}
\vspace{-1.5mm} %

  Theorem~\ref{thm:redefinition} provides a natural interpretation of
  the proposed approach. For $\lambda=1$, we have
  \begin{equation*}
    \alpha(1)
    =\beta(1)
    = \sum_{\omega\in\Omega}\min\left(\P(\omega), \Q(\omega)\right)
    = \sum_{\omega\in\Omega}\left[\P(\omega) - \left(\P(\omega) - \Q(\omega)\right)^+\right]
    = 1 - \delta(\P,\Q)
  \end{equation*}
  where $\delta(\P, \Q)$ denotes the total variation distance between $\P$ and
  $\Q$. As such, our notion of precision and recall may be viewed as a
  generalization of total variation distance.

\vspace{-1.0mm} %
\section{Application to Deep Generative Models}
\vspace{-1.5mm} %
\label{sec:application}

  In this section, we show that the algorithm introduced in
  Section~\ref{sec:algorithm} can be readily applied to evaluate precision
  and recall of deep generative models. In practice, access to $\P$ and $\Q$ is
  given via samples $\Ps\sim\P$ and $\Qs\sim\Q$. Given that both $\P$
  and $\Q$ are continuous distributions, the probability of generating a point
  sampled from $\Q$ is $0$. Furthermore, there is strong empirical evidence that comparing samples in image space runs the risk of assigning higher quality to a worse model~\cite{lopez2016revisiting,salimans2016improved,theis2015note}. A common remedy is to apply a pre-trained classifier trained on natural images and to compare $\Ps$ and $\Qs$ at a feature level. Intuitively, in this feature space the samples should be compared based on statistical regularities in the images rather than random artifacts resulting from the generative process~\cite{lopez2016revisiting,odena2016deconvolution}.

  Following this line of work, we first use a pre-trained Inception network to
  embed the samples (i.e. using the \emph{Pool3} layer~\cite{heusel2017gans}). We then
  cluster the union of $\Ps$ and $\Qs$ in this feature space using mini-batch k-means with $k=20$~\cite{sculley2010web}. Intuitively, we reduce the problem to a one dimensional problem where the histogram over the cluster assignments can be meaningfully compared. Hence, failing to produce samples from a cluster with many samples from the true distribution will hurt recall, and producing samples in clusters without many real samples will hurt precision. As the clustering algorithm is randomized, we run the procedure several times and average over the \coolname{} curves. We note that such a clustering is meaningful as shown in Figure~\ref{fig:cluster_vis} in the appendix and that it can be efficiently scaled to very large sample sizes~\cite{bachem2016approximate,bachem2016fast}.

  We stress that from the point of view of the proposed algorithm, only a meaningful embedding is required. As such, the algorithm can be applied to various data modalities. In particular, we show in Section~\ref{sec:mode_exp} that besides image data the algorithm can be applied to a text generation task.

\vspace{-1.0mm} %
\subsection{Adding and dropping modes from the target distribution}
\vspace{-1.5mm} %
\label{sec:mode_exp}

  \begin{figure}[t]
    \centering
    \includegraphics[width=0.378\linewidth]{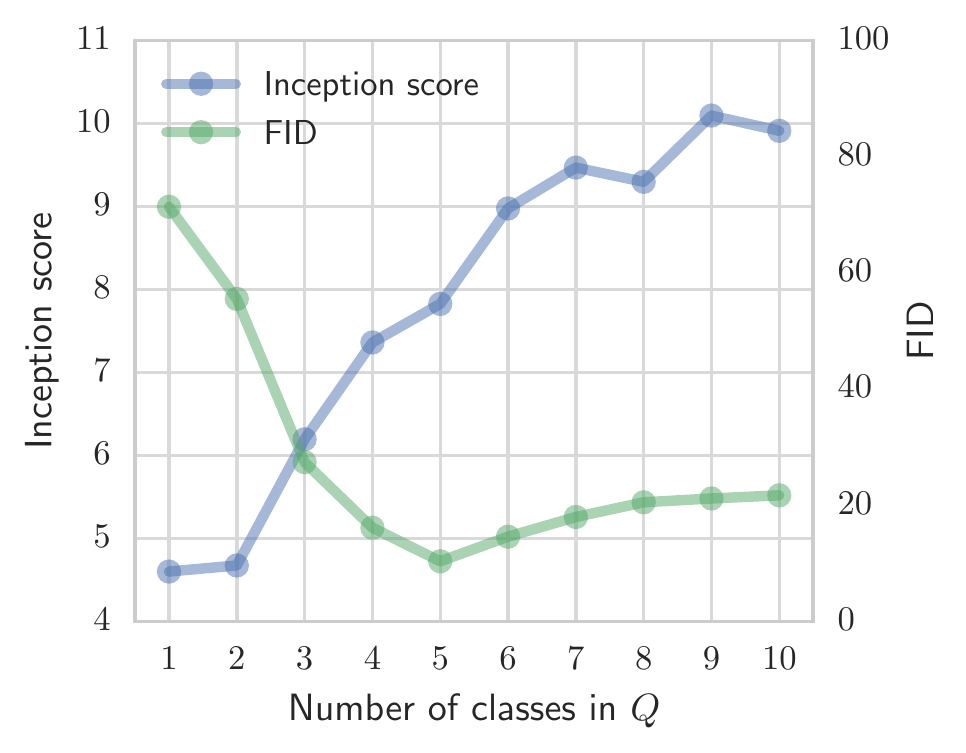}
    \includegraphics[width=0.3\linewidth]{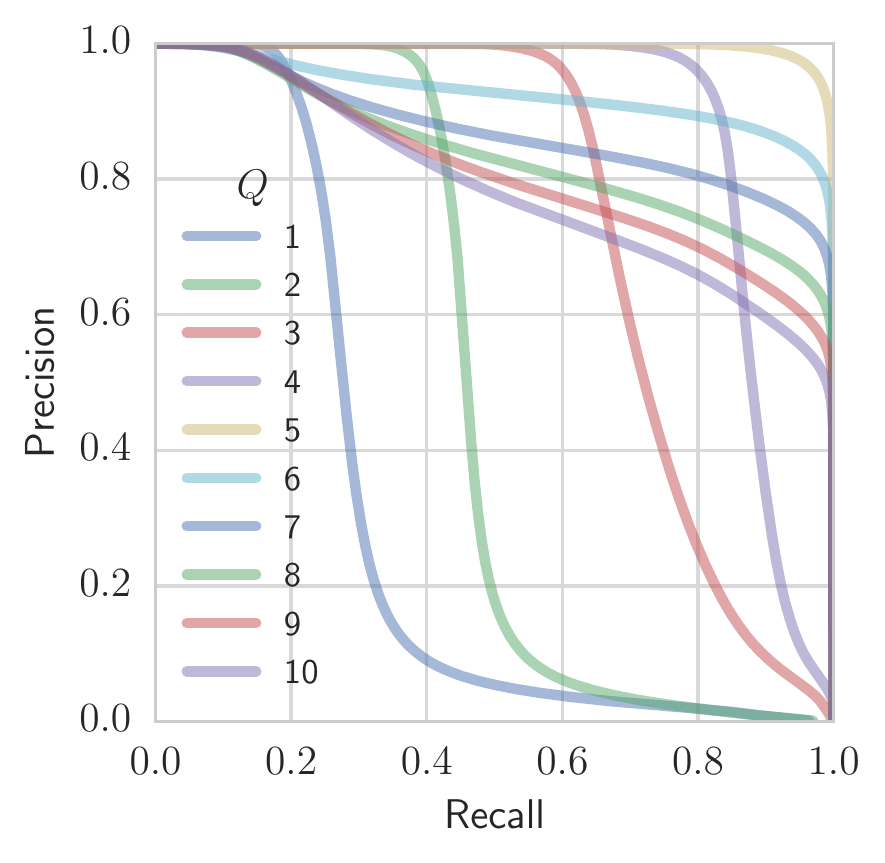}
    \includegraphics[width=0.3\linewidth]{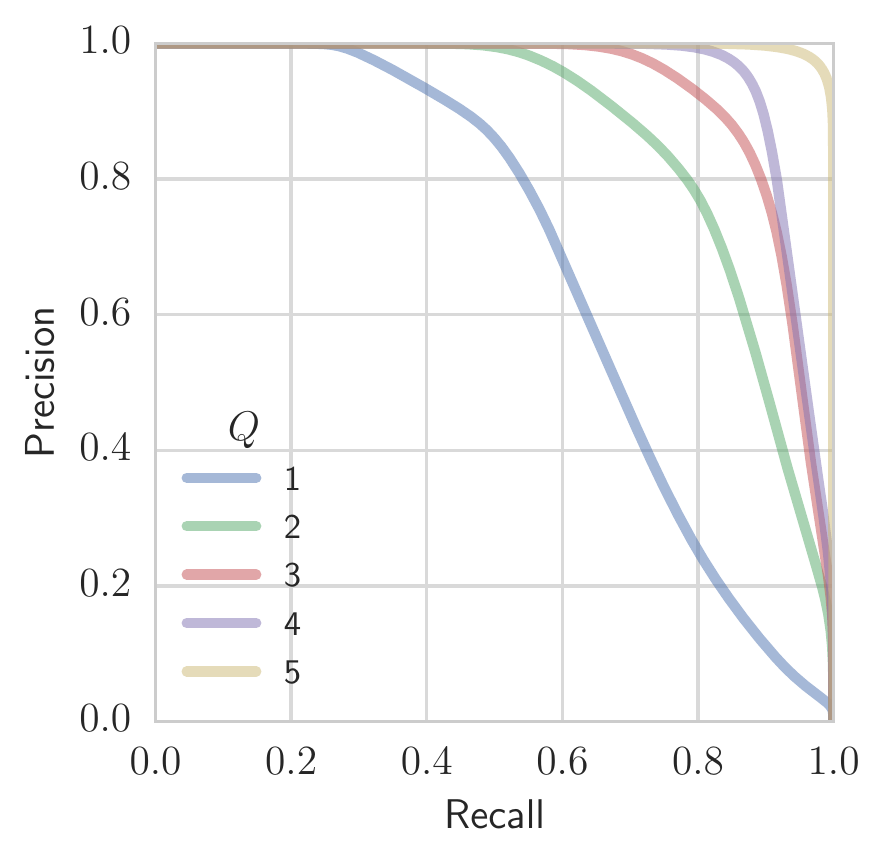}
    \caption{Left: IS and FID as we remove and add classes of CIFAR-10. IS generally only
    increases, while FID is sensitive to both the addition and removal of classes. However, it cannot distinguish between the two failure cases of inventing or dropping modes. Middle: Resulting \coolname{} curves for the same
    experiment. As expected, adding modes leads to a loss in precision ($\Q_6$--$\Q_{10}$), while dropping modes leads to a loss in recall ($\Q_1$--$\Q_4$). As an example consider $\Q_4$ and $\Q_6$ which have similar FID, but strikingly different \coolname{} curves. The same behavior can be observed for the task of text generation, as displayed on the plot on the right.
    For this experiment, we set $\P$ to contain samples from all classes so the \coolname{} curves demonstrate the increase in recall as we increase the number of classes in $\Q$.}
    \vspace{-4mm}
    \label{fig:modes_exp}
  \end{figure}

  Mode collapse or mode dropping is a major challenge in
  GANs~\cite{goodfellow2014generative, salimans2016improved}. Due to the
  symmetry of commonly used metrics with respect to precision and recall, the
  only way to assess whether the model is producing low-quality images or
  dropping modes is by visual inspection. In stark contrast, the proposed metric can quantitatively disentangle these effects which we empirically demonstrate.

  We consider three data sets commonly used in the GAN literature:
  MNIST~\cite{mnist}, Fashion-MNIST~\cite{fashionmnist}, and
  CIFAR-10~\cite{cifar10}. These data sets are labeled and consist of 10
  balanced classes. To show the sensitivity of the proposed measure to mode
  dropping and mode inventing, we first fix $\Ps$ to contain samples from the first 5 classes in the respective test set. Then, for a fixed $i=1,\dots,10$, we
  generate a set $\Qs_i$, which consists of samples from the first $i$ classes
  from the training set. As $i$ increases, $\Qs_i$ covers an increasing number of
  classes from $\Ps$ which should result in higher recall. As we
  increase $i$ beyond 5, $\Qs_i$ includes samples from an increasing number of
  classes that are not present in $\Ps$ which should result in a loss in precision, but not in recall as the other classes are already covered. Finally, the set $\Qs_5$ covers the same classes as $\Ps$, so it should have high precision and high recall.

  Figure~\ref{fig:modes_exp} (left) shows the IS and FID for the CIFAR-10
  data set (results on the other data sets are shown in
  Figure~\ref{fig:modes_exp_appendix} in the appendix). Since the IS is not computed \wrt a reference
  distribution, it is invariant to the choice of $\Ps$, so as we add classes to
  $\Qs_i$, the IS increases. The FID decreases as we
  add more classes until $\Qs_5$ before it starts to increase as we add spurious
  modes. Critically, FID fails to distinguish the cases of mode dropping and mode inventing: $\Qs_4$ and $\Qs_6$ share similar FIDs. In contrast, Figure~\ref{fig:modes_exp} (middle) shows our \coolname{} curves as we vary the number of classes in $\Qs_i$. Adding correct modes leads to an increase in recall, while adding fake modes leads to a loss of precision.

  We also apply the proposed approach on text data as shown in Figure~\ref{fig:modes_exp} (right). In particular, we use the {MultiNLI}
  corpus of crowd-sourced sentence pairs annotated
  with topic and textual entailment
  information~\cite{williams2017broad}. After discarding the entailment label, we
  collect all unique sentences for the same topic.
  Following~\cite{cifka2018eval}, we embed these sentences using a BiLSTM with
  2048 cells in each direction and max pooling, leading to a 4096-dimensional
  embedding~\cite{conneau2017supervised}. We consider 5 classes from this data set and fix $\Ps$ to contain samples from all classes to measure the loss in recall for different $\Q_i$. Figure~\ref{fig:modes_exp} (right) curves successfully demonstrate the sensitivity of recall to mode dropping.

\vspace{-1.0mm} %
\subsection{Assessing class imbalances for GANs}
\vspace{-1.5mm} %
\label{sec:mnistinceptionscore}

  \begin{figure}[t]
    \centering
    \begin{center}
      \setlength{\tabcolsep}{2pt}
      \begin{tabular}{cc}
        \begin{tabular}{c}
          \includegraphics[trim={0 112 0 0},clip,width=0.28\linewidth]{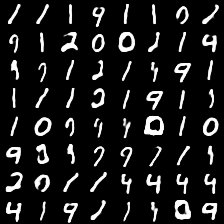} \\[-1mm]
          \includegraphics[trim={0 0 0 0},clip,width=0.28\linewidth]{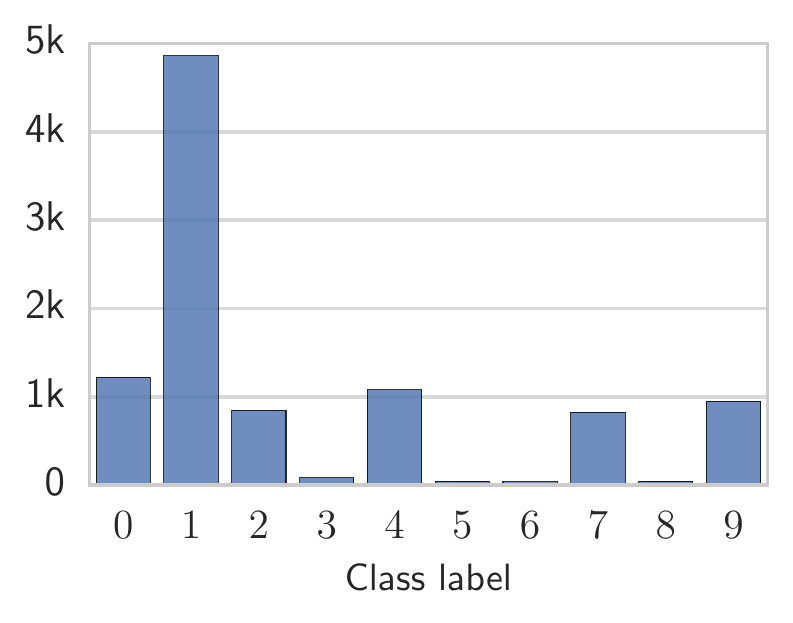}
        \end{tabular}
        \begin{tabular}{c}
          \vspace{3mm}
          \includegraphics[width=0.39\linewidth]{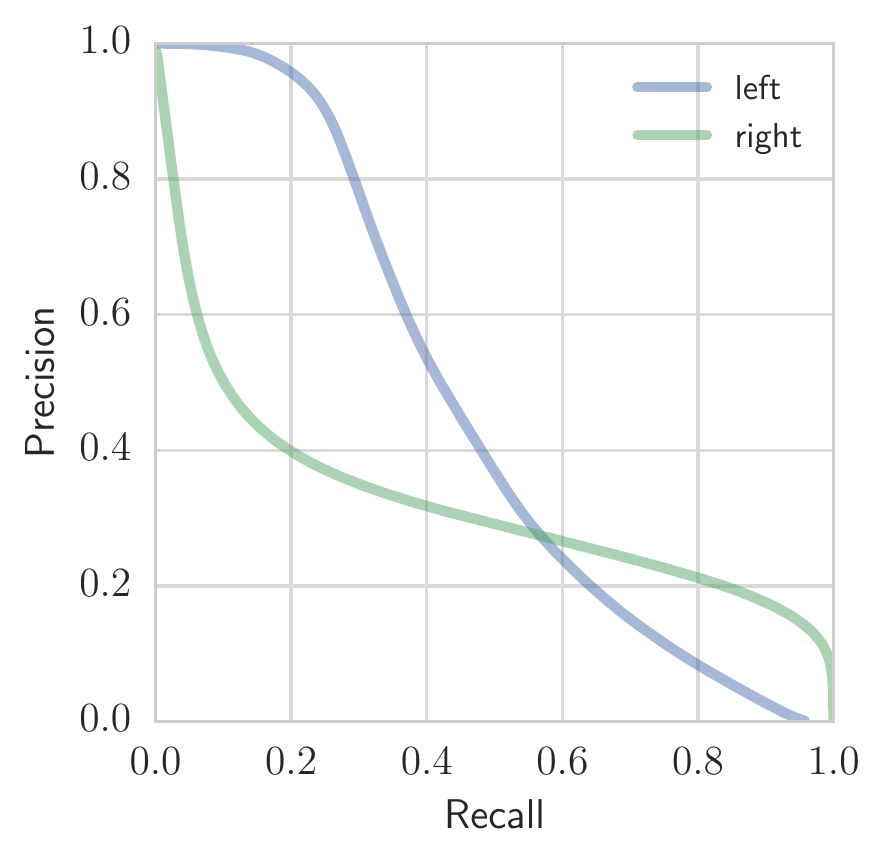}
        \end{tabular}
        \begin{tabular}{c}
          \includegraphics[trim={0 112 0 0},clip,width=0.28\linewidth]{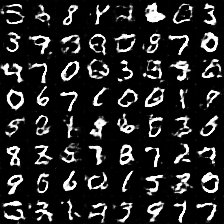} \\[-1mm]
          \includegraphics[trim={0 0 0 0},clip,width=0.28\linewidth]{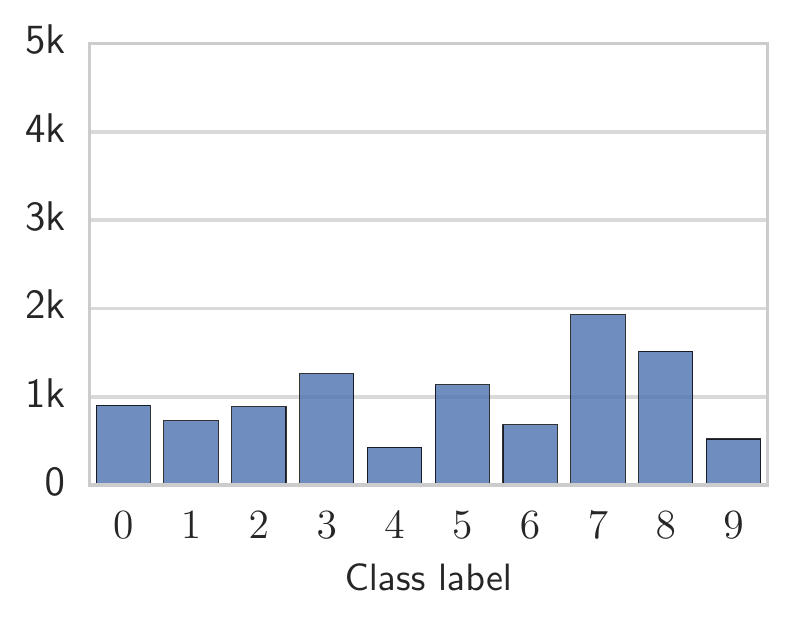}
        \end{tabular}
      \end{tabular}
    \end{center}
    \hfill
    \vspace{-6mm}
    \caption{Comparing two GANs trained on MNIST which both achieve an FID of 49. The model on the left seems to produce high-quality samples of only a subset of digits. On the other hand, the model on the right generates low-quality samples of all digits. The histograms showing the corresponding class distributions based on a trained MNIST classifier confirm this observation. At the same time, the classifier is more confident which indicates different levels of precision (96.7\% for the model on the left compared to 88.6\% for the model on the right).
    Finally, we note that the proposed \coolname{} algorithm does not require
    labeled data, as opposed to the IS which further needs a classifier that was trained on the respective data set.}
    \vspace{-2mm}
    \label{fig:mnist_classifier}
  \end{figure}

  In this section we analyze the effect of class imbalance on the \coolname{} curves.
  Figure~\ref{fig:mnist_classifier} shows a pair of GANs trained on MNIST which
  have virtually the same FID, but very different \coolname{} curves.
  The model on the left generates a subset of the digits of high quality, while the model on the right seems to generate all digits, but each has low quality. We can naturally interpret this difference via the \coolname{} curves: For a desired recall level of less than ${\sim}0.6$, the model on the left enjoys higher precision -- it generates several digits of high quality. If, however, one desires a recall higher than ${\sim}0.6$, the model on the right enjoys higher precision as it covers all digits.
  To confirm this, we train an MNIST classifier on the embedding of $\Ps$ with the ground truth labels and plot the distribution of the predicted classes for both models. The histograms clearly show that the model on the left failed to generate all classes (loss in recall), while the model on the right is producing a more balanced distribution over all classes (high
  recall). At the same time, the classifier has an average
  \emph{confidence}\footnote{We denote the output of the classifier for its
  highest value at the softmax layer as confidence. The intuition is that
  higher values signify higher confidence of the model for the given label.} of
  96.7\% on the model on the left compared to 88.6\% on the model on the right,
  indicating that the sample quality of the former is higher. This aligns very
  well with the \coolname{} plots: samples on the left have high quality but are not diverse in contrast to the samples on the right which are diverse but have low quality.

  This analysis reveals a connection to IS which is based on the premise that the conditional label distribution $p(y|x)$ should have low entropy, while the marginal $p(y)=\int p(y|x = G(z))dz$ should have high
  entropy. To further analyze the relationship between the proposed approach and \coolname{} curves, we plot $p(y|x)$ against precision and $p(y)$
  against recall in Figure~\ref{fig:inception_score_scatter} in the
  appendix. The results over a large number of GANs and VAEs show a large
  Spearman correlation of -0.83 for precision and 0.89 for recall. We however stress two key differences between the approaches: Firstly, to compute the quantities in IS one needs a classifier and a labeled data set in contrast to the proposed \coolname{} metric which can be applied on unlabeled data. Secondly, IS only captures losses in recall \wrt classes, while our metric measures more fine-grained recall losses (see Figure~\ref{fig:mnist_tilted} in the appendix).

\vspace{-1.0mm} %
\subsection{Application to GANs and VAEs}
\vspace{-1.5mm} %
\label{sec:applicationtogans}

  \tikzset{
      image label/.style={
          every node/.style={
              fill=white,
              text=black,
              font=\fontfamily{phv}\selectfont\scriptsize\bfseries,
              anchor=north west,
              xshift=0mm,
              yshift=0mm,
              at={(0,1)}
          }
      }
  }

  \begin{figure}[t]
    \begin{center}
      \setlength{\tabcolsep}{2pt}
      \begin{tabular}{cc}
        \begin{tabular}{cc}
          \includegraphics[width=0.485\linewidth]{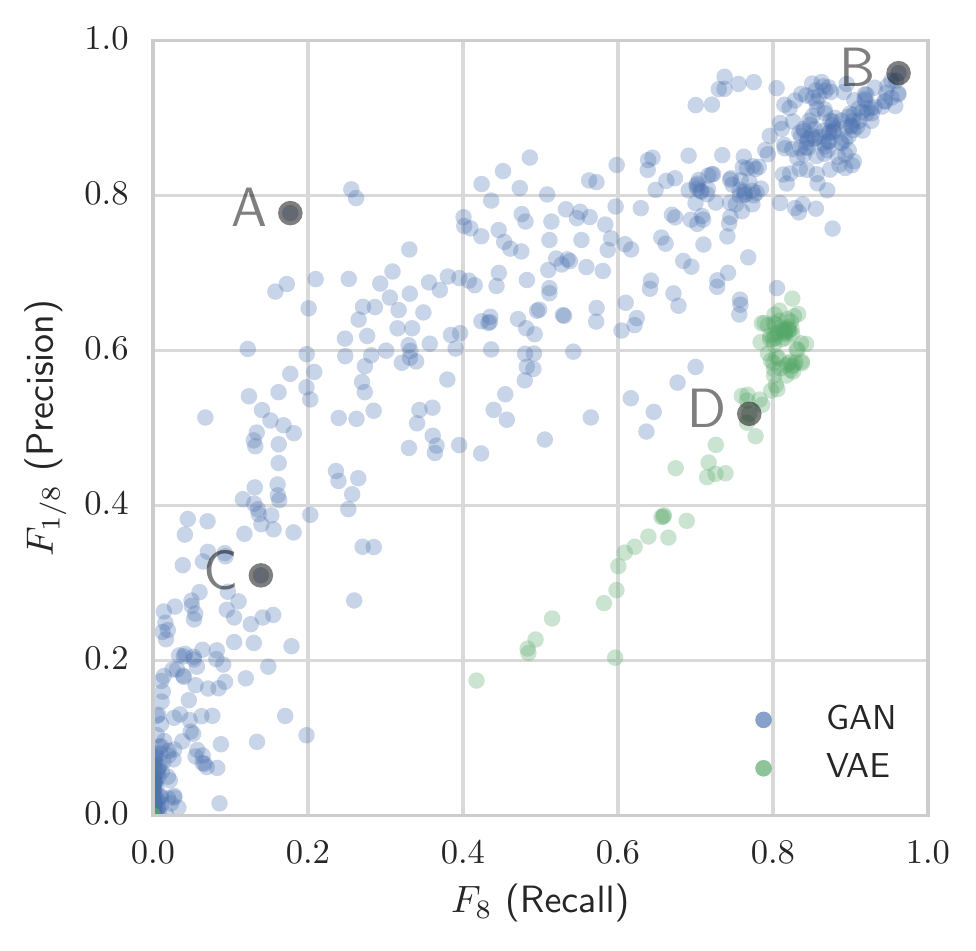}
        \end{tabular}
        \begin{tabular}{cc}
          \begin{tikzonimage}[width=0.21\textwidth]{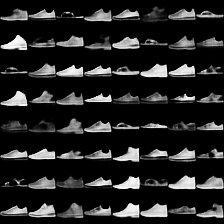}[image label]
              \node{A};
          \end{tikzonimage} &
          \begin{tikzonimage}[width=0.21\textwidth]{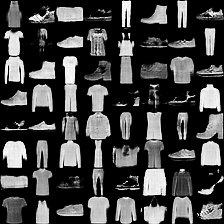}[image label]
              \node{B};
          \end{tikzonimage} \\
          \begin{tikzonimage}[width=0.21\textwidth]{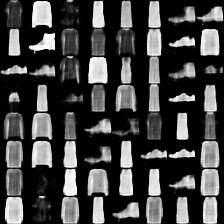}[image label]
              \node{C};
          \end{tikzonimage} &
          \begin{tikzonimage}[width=0.21\textwidth]{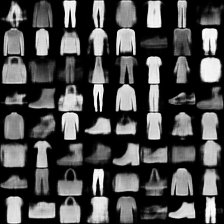}[image label]
              \node{D};
          \end{tikzonimage} \\
        \end{tabular}
      \end{tabular}
    \end{center}
    \hfill
    \vspace{-3mm}
    \caption{$F_{1/8}$ vs $F_8$ scores for a large number of GANs and VAEs on the Fashion-MNIST data set. For each model, we plot the maximum $F_{1/8}$ and $F_8$ scores to show the trade-off between precision and recall. VAEs generally achieve lower precision and/or higher recall than GANs which matches the folklore that VAEs often produce samples of lower quality while being less prone to mode collapse. On the right we show samples from four models which correspond to various success/failure modes: (A) high precision, low recall, (B) high precision, high recall, (C) low precision, low recall, and (D) low precision, high recall.}
    \label{fig:all_gans_2d}
  \end{figure}

  We evaluate the precision and recall of 7 GAN types and the VAE with 100 hyperparameter settings each as provided by~\cite{lucic2017gans}. In order to visualize this vast quantity of models, one needs to summarize the \coolname{} curves. A natural idea is to compute the maximum $F_1$ score, which corresponds to the harmonic mean between precision and recall as a single-number summary. This idea is fundamentally flawed as $F_1$ is symmetric. However, its generalization,
  defined as $F_\beta=(1+\beta^2)\frac{p\cdot r}{(\beta^2p)+r}$, provides a
  way to quantify the relative importance of precision and recall: $\beta > 1$ weighs recall higher than precision, whereas $\beta < 1$ weighs precision higher than recall. As a result, we propose to distill each \coolname{} curve into a pair of values: $F_\beta$ and $F_{1/\beta}$.

  Figure~\ref{fig:all_gans_2d} compares the maximum $F_8$ with the maximum $F_{1/8}$ for
  these models on the Fashion-MNIST data set. We choose $\beta=8$ as it offers a good insight into the bias towards precision versus recall. Since $F_8$
  weighs recall higher than precision and $F_{1/8}$ does the opposite, models
  with higher recall than precision will lie below the diagonal $F_8=F_{1/8}$
  and models with higher precision than recall will lie above. To our knowledge, this is the first metric which confirms the folklore that VAEs are biased towards higher recall, but may suffer from precision issues (\eg, due to blurring effects), at least on this data set. On the right, we show samples from four models on the extreme ends of the plot for all combinations of high and low precision and recall. We have included similar plots on the MNIST, CIFAR-10 and CelebA data sets in the appendix.

\vspace{-1.0mm} %
\section{Conclusion}\label{sec:conclusion}
\vspace{-1.5mm} %

Quantitatively evaluating generative models is a challenging task of paramount importance. %
In this work we show that one-dimensional scores are not sufficient to capture different failure cases of current state-of-the-art generative models. As an alternative, we propose a novel notion of precision and recall for distributions and prove that both notions are theoretically sound and have desirable properties. We then connect these notions to total variation distance as well as FID and IS and we develop an efficient algorithm that can be readily applied to evaluate deep generative models based on samples. We investigate the properties of the proposed algorithm on real-world data sets, including image and text generation, and show that it captures the precision and recall of generative models. Finally, we find empirical evidence supporting the folklore that VAEs produce samples of lower quality, while being less prone to mode collapse than GANs.

\clearpage
{
\small
\bibliographystyle{plainnat}
\bibliography{paper.bib}
}

\clearpage
\appendix

\section{Proofs}

  We first show the following auxiliary result and then prove
  Theorems~\ref{thm:properties} and~\ref{thm:redefinition}.
  \begin{lemma}
    \label{lem:inequality}
    Let $\P$ and $\Q$ be probability distributions defined on a finite state
    space $\Omega$. Let $\alpha\in(0,1]$ and $\beta\in(0,1]$. Then, $(\alpha,
    \beta)\in\PR$ if and only if there exists a distribution $\RR$ such that for
    all $\omega\in\Omega$
    \begin{equation}
      \label{eqn:prinequality}
      \P(\omega) \geq \beta \RR(\omega)\quad\text{and}\quad \Q(\omega)\geq \alpha\RR(\omega).
    \end{equation}
  \end{lemma}
  \begin{proof}
    If $(\alpha, \beta)\in\PR$, then~\eqref{eqn:prdefinition} and the
    non-negativity of $\mP$ and $\mQ$ directly imply~\eqref{eqn:prinequality}
    for the same choice of $\RR$. Conversely, if~\eqref{eqn:prinequality} holds
    for a distribution $\mu$, we may define the distributions
    \begin{equation*}
      \mP(\omega) = \frac{\P(\omega)-\beta\RR(\omega)}{1-\beta}
      \quad\text{and}\quad
      \mQ(\omega) = \frac{\Q(\omega)-\alpha\RR(\omega)}{1-\alpha}.
    \end{equation*}
    By definition $\alpha$, $\beta$, $\RR$, $\mP$ and $\mQ$
    satisfy~\eqref{eqn:prdefinition} in Definition~\ref{def:precisionrecall}
    which implies $(\alpha, \beta)\in\PR$.
  \end{proof}

\subsection{Proof of Theorem~\ref{thm:properties}}
\label{sec:app:proof:properties}

  \begin{proof}
    We show each of the properties independently:

    \emph{\ref{prop:equality} Equality}:
      If $(1,1) \in \PR$, then we have by Definition~\ref{def:precisionrecall}
      that $\P=\RR$ and $\Q=\RR$ which implies $\P=\Q$ as claimed. Conversely,
      if $\P=\Q$, Definition~\ref{def:precisionrecall} is satisfied for
      $\alpha=\beta=1$ by choosing $\RR=\mP=\mQ=\P$. Hence, $(1,1) \in \PR$
      as claimed.

    \emph{\ref{prop:disjoint} Disjoint support}:
      We show both directions of the claim by contraposition, \ie, we show
      $\supp(\P) \cap\supp(\Q) \neq \emptyset \;\Leftrightarrow\; \PR \supset
      \lbrace(0,0)\rbrace$. Consider an arbitrary $\omega\in\supp(\P)
      \cap\supp(\Q)$. Then, by definition we have $\P(\omega) > 0$ and
      $\Q(\omega)>0$. Let $\RR$ be defined as the distribution with
      $\RR(\omega)=1$ and $\RR(\omega') = 0$ for all
      $\omega'\in\Omega\setminus\lbrace\omega\rbrace$. Clearly, it holds that
      $\P(\omega)\geq \P(\omega)\RR(\omega)$ and $\Q(\omega)\geq
      \Q(\omega)\RR(\omega)$ for all $\omega\in\Omega$. Hence, by
      Lemma~\ref{lem:inequality}, we have  $\left(\Q(\omega),
      \P(\omega)\right)\in\PR$ which implies that $\PR \supset
      \lbrace(0,0)\rbrace$ as claimed. Conversely, $\PR \supset
      \lbrace(0,0)\rbrace$ implies by Lemma~\ref{lem:inequality} that there
      exist $\alpha\in(0,1]$ and $\beta\in(0,1]$ as well as a distribution $\RR$
      satisfying~\eqref{eqn:prinequality}. Let $\omega\in\supp(\RR)$ which
      implies $\RR(\omega)>0$ and thus by~\eqref{eqn:prinequality} also
      $\P(\omega) > 0$ and $\Q(\omega)>0$. Hence, $\omega$ is both in the
      support of $\P$ and $\Q$ which implies $\supp(\P) \cap\supp(\Q) \neq
      \emptyset$ as claimed.

    \emph{\ref{prop:maxprecision} Maximum precision}:
      If $(\alpha,\beta)\in \PR$, then by Lemma~\ref{lem:inequality} there
      exists a distribution $\RR$ such that for all $\omega\in\Omega$ we have
      $\P(\omega) \geq \beta \RR(\omega)$ and $\Q(\omega)\geq
      \alpha\RR(\omega)$. $\P(\omega) \geq \beta \RR(\omega)$  implies
      $\supp(\RR)\subseteq\supp(\P)$ and hence
      $\sum_{\omega\in\supp(\P)}\RR(\omega)=1$. Together with $\Q(\omega)\geq
      \alpha\RR(\omega)$, this yields
      $\Q(\supp(\P))=\sum_{\omega\in\supp(\P)}\Q(\omega) \geq
      \alpha\sum_{\omega\in\supp(\P)}\RR(\omega) = \alpha$ which implies
      $\alpha\leq \Q(\supp(\P))$ for all $(\alpha,\beta) \in \PR$.

      To prove the claim, we next show that there exists $(\alpha,\beta) \in
      \PR$ with $\alpha= \Q(\supp(\P))$. Let $\iPQ=\supp(\P)\cap\supp(\Q)$. If
      $\iPQ=\emptyset$, then $\alpha=\Q(\supp(\P))=0$ and $(0,0)\in\PR$ by
      Definition~\ref{def:precisionrecallset} as claimed. For the case
      $\iPQ\neq\emptyset$, let
      $\beta=\min_{\omega\in\iPQ}\frac{\P(\omega)\Q(\iPQ)}{\Q(\omega)}$. By
      definition of $\iPQ$, we have $\beta>0$. Furthermore, $\beta \leq
      \P(\iPQ)\leq 1$ since $\frac{\P(\omega)}{\P(\iPQ)}\leq
      \frac{\Q(\omega)}{\Q(\iPQ)}$ for at least one $\omega\in\iPQ$. Consider
      the distribution $\RR$ where $\RR(\omega)= \frac{\Q(\omega)}{\Q(\iPQ)}$
      for all $\omega\in\iPQ$ and $\RR(\omega)=0$ for
      $\omega\in\Omega\setminus\iPQ$. By construction, $\RR$
      satisfies~\eqref{eqn:prinequality} in Lemma~\ref{lem:inequality} and hence
      $(\alpha,\beta)\in\PR$ as claimed.

    \emph{\ref{prop:maxrecall} Maximum recall}:
      This follows directly from applying Property~\ref{prop:duality} to
      Property~\ref{prop:maxprecision}.

    \emph{\ref{prop:monotonicity} Monotonicity}:
      If $(\alpha,\beta)\in \PR$, then by Lemma~\ref{lem:inequality} there
      exists a distribution $\RR$ such that for all $\omega\in\Omega$ we have
      that $\P(\omega) \geq \beta \RR(\omega)$ and $\Q(\omega)\geq
      \alpha\RR(\omega)$. For $\alpha'\in(0,\alpha]$ and $\beta'\in(0,\beta]$,
      it follows that $\P(\omega) \geq \beta' \RR(\omega)$ and $\Q(\omega)\geq
      \alpha'\RR(\omega)$ for all $\omega\in\Omega$. By
      Lemma~\ref{lem:inequality} this implies $(\alpha',\beta')\in \PR$ as
      claimed.

    \emph{\ref{prop:duality} Duality}:
      This follows directly from switching $\alpha$ with $\beta$, $\P$ with $\Q$
      and $\mP$ with $\mQ$ in Definition~\ref{def:precisionrecall}.

  \end{proof}

\newpage
\subsection{Proof of Theorem~\ref{thm:redefinition}}
\label{sec:app:proof:redefinition}

  \begin{proof}
    We first show that $\PR\subseteq\left\lbrace\left(\theta\alpha(\lambda),
    \theta\beta(\lambda)\right)\mid\lambda\in(0,\infty),
    \theta\in[0,1]\right\rbrace$. We consider an arbitrary element $(\alpha',
    \beta')\in\PR$ and show that $(\alpha', \beta') =
    \left(\theta\alpha(\lambda), \theta\beta(\lambda)\right)$ for some
    $\lambda\in(0,\infty)$ and $\theta\in[0,1]$. For the case $(\alpha',
    \beta')=(0,0)$, the result holds trivially for the choice of $\lambda=1$ and
    $\theta=0$. For the case $(\alpha', \beta')\neq(0,0)$, we choose
    $\lambda=\frac{\alpha'}{\beta'}$ and $\theta=\frac{\beta'}{\beta(\lambda)}$.
    Since $\alpha(\lambda)=\lambda\beta(\lambda)$ by definition, this implies
    $(\alpha', \beta') = \left(\theta\alpha(\lambda),
    \theta\beta(\lambda)\right)$ as required. Furthermore,
    $\lambda\in(0,\infty)$ since by Definitions~\ref{def:precisionrecall}
    and~\ref{def:precisionrecallset} $\alpha'>0$ if and only if $\beta'>0$.
    Similarly, we show that $\theta\in[0,1]$: By Lemma~\ref{lem:inequality}
    there exists a distribution $\RR$ such that $\beta'
    \RR(\omega)\leq\P(\omega)$ and $\alpha'\RR(\omega)\leq\Q(\omega)$ for all
    $\omega\in\Omega$. This implies that
    $\beta'\RR(\omega)\leq\frac{\Q(\omega)}{\lambda}$ and thus
    $\beta'\RR(\omega)\leq\min\left(\P(\omega),\frac{\Q(\omega)}{\lambda}\right)$
    for all $\omega\in\Omega$. Summing over all $\omega\in\Omega$, we obtain
    $\beta' \leq
    \sum_{\omega\in\Omega}\min\left(\P(\omega),\frac{\Q(\omega)}{\lambda}\right)=\beta(\lambda)$
    which implies $\theta\in[0,1]$.

    Finally, we show that $\PR\supseteq\left\lbrace\left(\theta\alpha(\lambda),
    \theta\beta(\lambda)\right)\mid\lambda\in(0,\infty),
    \theta\in[0,1]\right\rbrace$. Consider arbitrary $\lambda\in(0,\infty)$ and
    $\theta\in[0,1]$. If $\beta(\lambda)=0$, the claim holds trivially since
    $(0,0)\in\PR$. Otherwise, define the distribution $\RR$ by $
    \RR(\omega) =
    \min\left(\P(\omega),\frac{\Q(\omega)}{\lambda}\right)/\beta(\lambda)$ for
    all $\omega\in\Omega$. By definition, $\beta(\lambda)\RR(\omega) \leq
    \min\left(\P(\omega),\frac{\Q(\omega)}{\lambda}\right)\leq\P(\omega)$ for
    all $\omega\in\Omega$. Similarly, $\alpha(\lambda)\RR(\omega) \leq
    \min\left(\lambda\P(\omega),\Q(\omega)\right)\leq\Q(\omega)$ for all
    $\omega\in\Omega$ since $\alpha(\lambda)=\lambda\beta(\lambda)$. Because
    $\theta\in[0,1]$, this implies $\theta\beta(\lambda)\RR(\omega)
    \leq\P(\omega)$ and $\theta\alpha(\lambda)\RR(\omega) \leq\Q(\omega)$ for
    all $\omega\in\Omega$. Hence, by Lemma~\ref{lem:inequality},
    $\left(\theta\alpha(\lambda), \theta\beta(\lambda)\right)\in\PR$ for all
    $\lambda\in(0,\infty)$ and $\theta\in[0,1]$ as claimed.
  \end{proof}

\section{Further figures}

  \begin{figure}[h]
    \centering
    \begin{center}
      \setlength{\tabcolsep}{2pt}
      \begin{tabular}{cc}
        \begin{tabular}{c}
          \includegraphics[trim={0 112 0     0},clip,width=0.28\linewidth]{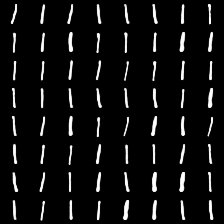} \\[-1mm]
          \includegraphics[trim={0 0 0 0},clip,width=0.28\linewidth]{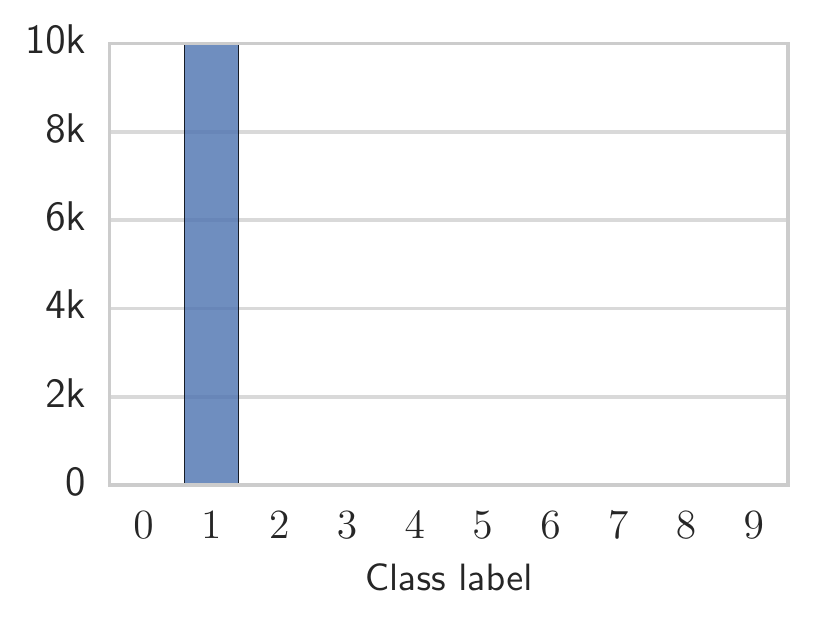}
        \end{tabular}
        \begin{tabular}{c}
          \vspace{3mm}
          \includegraphics[width=0.39\linewidth]{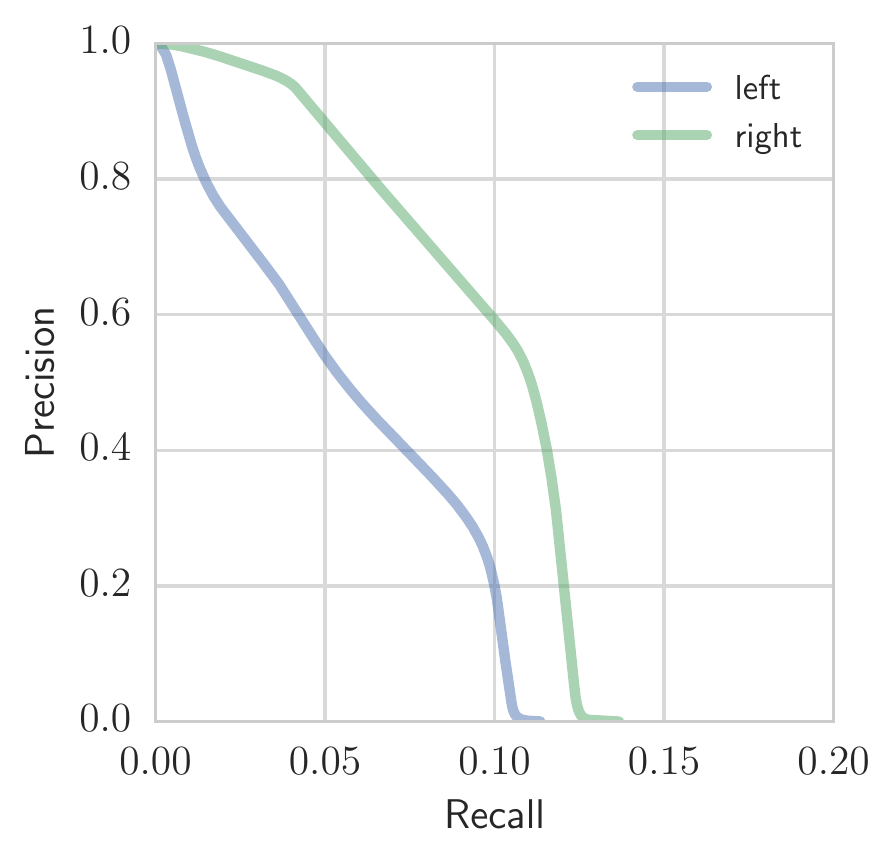}
        \end{tabular}
        \begin{tabular}{c}
          \includegraphics[trim={0 112 0 0},clip,width=0.28\linewidth]{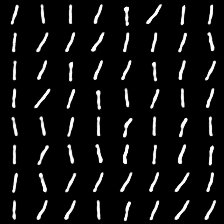} \\[-1mm]
          \includegraphics[trim={0 0 0 0},clip,width=0.28\linewidth]{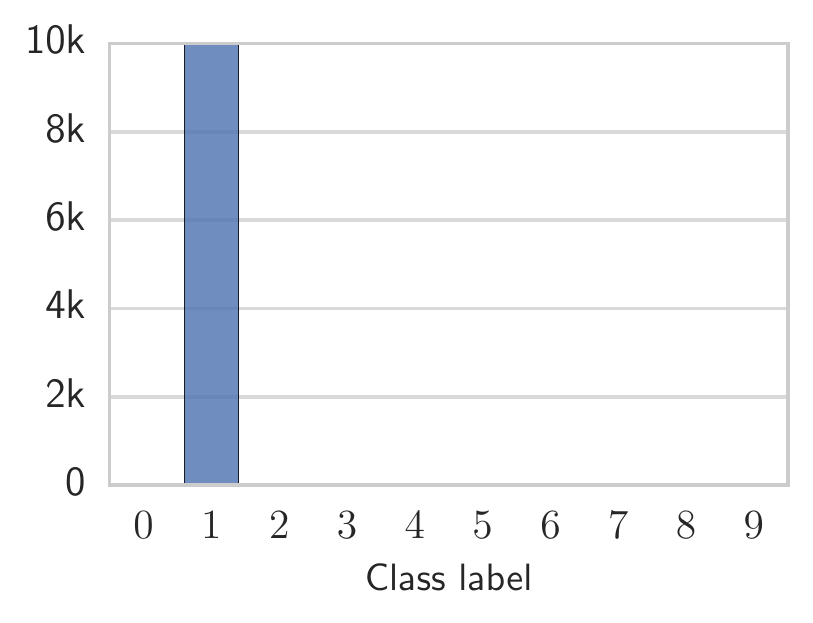}
        \end{tabular}
      \end{tabular}
    \end{center}
    \hfill
    \vspace{-3mm}
    \caption{Comparing a pair of GANs on MNIST which have both collapsed to
    producing 1's. An analysis with a trained classifier as in
    Section~\ref{sec:mnistinceptionscore} comes to the same conclusion for both
    models, namely, that they have collapsed to producing 1's only. However, the
    \coolname{} curve shows that the model on the right has a slightly higher recall.
    This is indeed correct: while the model on the left is producing straight
    1's only, the model on the right is producing some more varied shapes such
    as tilted 1's.}
    \label{fig:mnist_tilted}
  \end{figure}

\newpage

  \begin{figure}[t]
    \centering
    \setlength{\tabcolsep}{2pt}
    \begin{center}
    \begin{tabular}{cc}
      Real images & Generated images \\
      \includegraphics[trim={0 0 0 0},clip,width=0.48\linewidth]{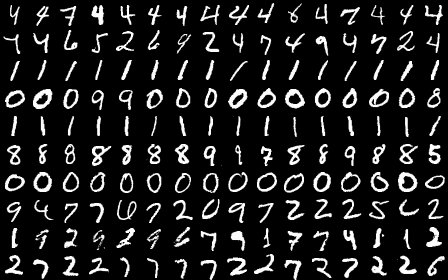} &
      \includegraphics[trim={0 0 0 0},clip,width=0.48\linewidth]{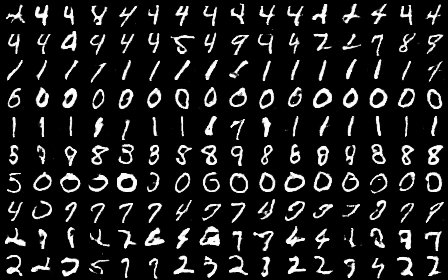} \\
      \includegraphics[trim={0 0 0 0},clip,width=0.48\linewidth]{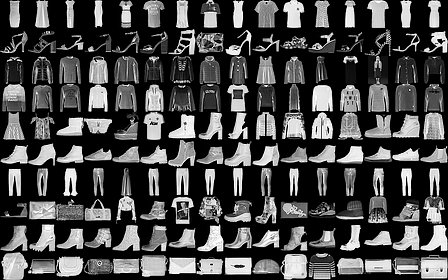} &
      \includegraphics[trim={0 0 0 0},clip,width=0.48\linewidth]{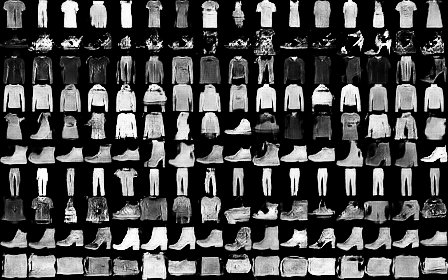} \\
      \includegraphics[trim={0 0 0 0},clip,width=0.48\linewidth]{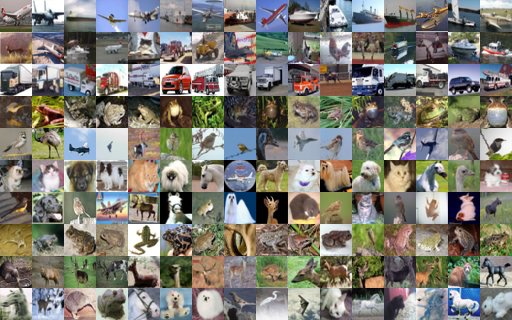} &
      \includegraphics[trim={0 0 0 0},clip,width=0.48\linewidth]{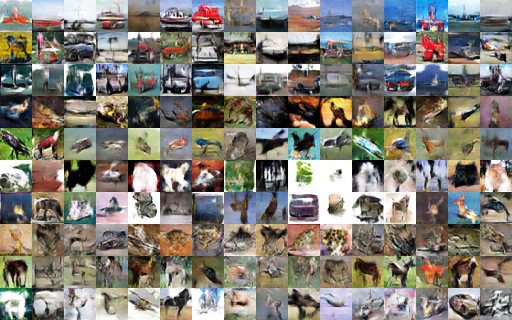} \\
      \includegraphics[trim={0 0 0 0},clip,width=0.48\linewidth]{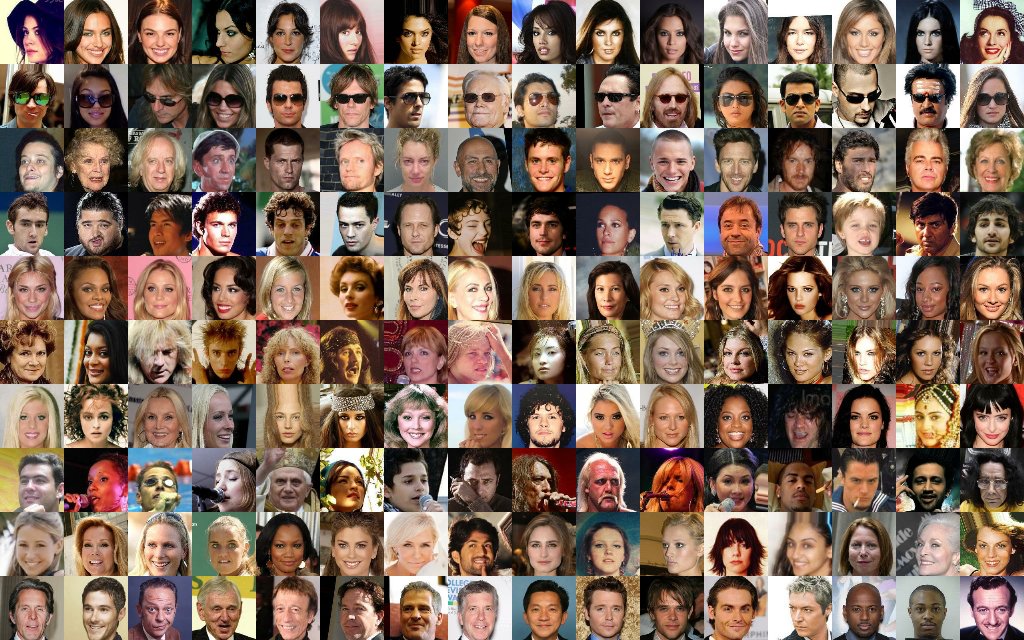} &
      \includegraphics[trim={0 0 0 0},clip,width=0.48\linewidth]{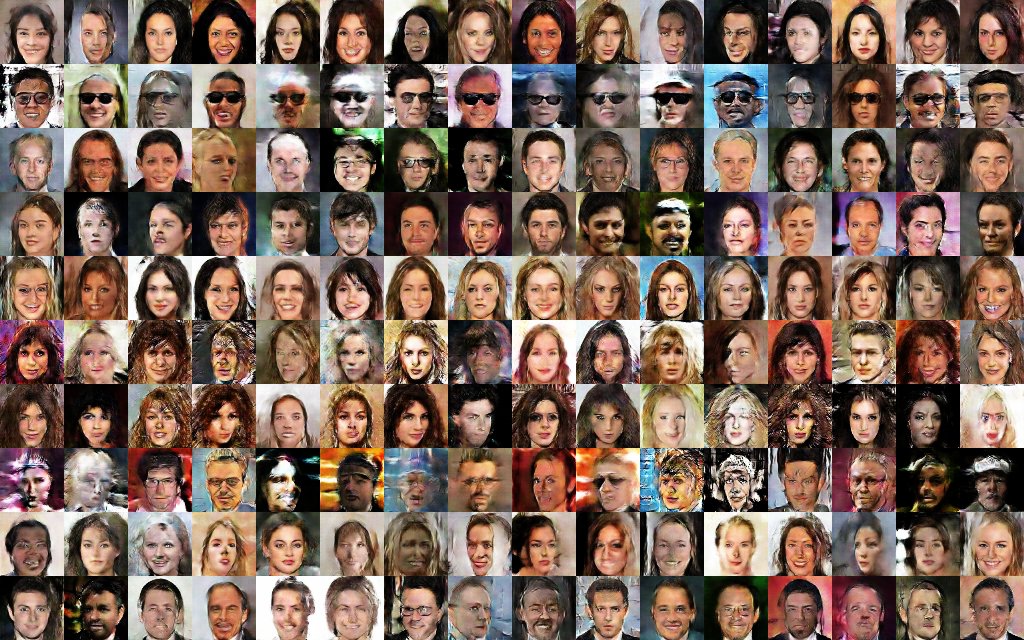} \\
    \end{tabular}
    \end{center}
    \vspace{-3mm}
    \caption{Clustering the real and generated samples from a GAN in feature
    space (10 cluster centers for visualization) yields the clusters above for
    the data sets MNIST, Fashion-MNIST, CIFAR-10 and CelebA. Although the GAN
    samples are not perfect, they are clustered in a meaningful way.}
    \label{fig:cluster_vis}
  \end{figure}

  \begin{figure}[t]
    \centering
    \includegraphics[width=0.45\linewidth]{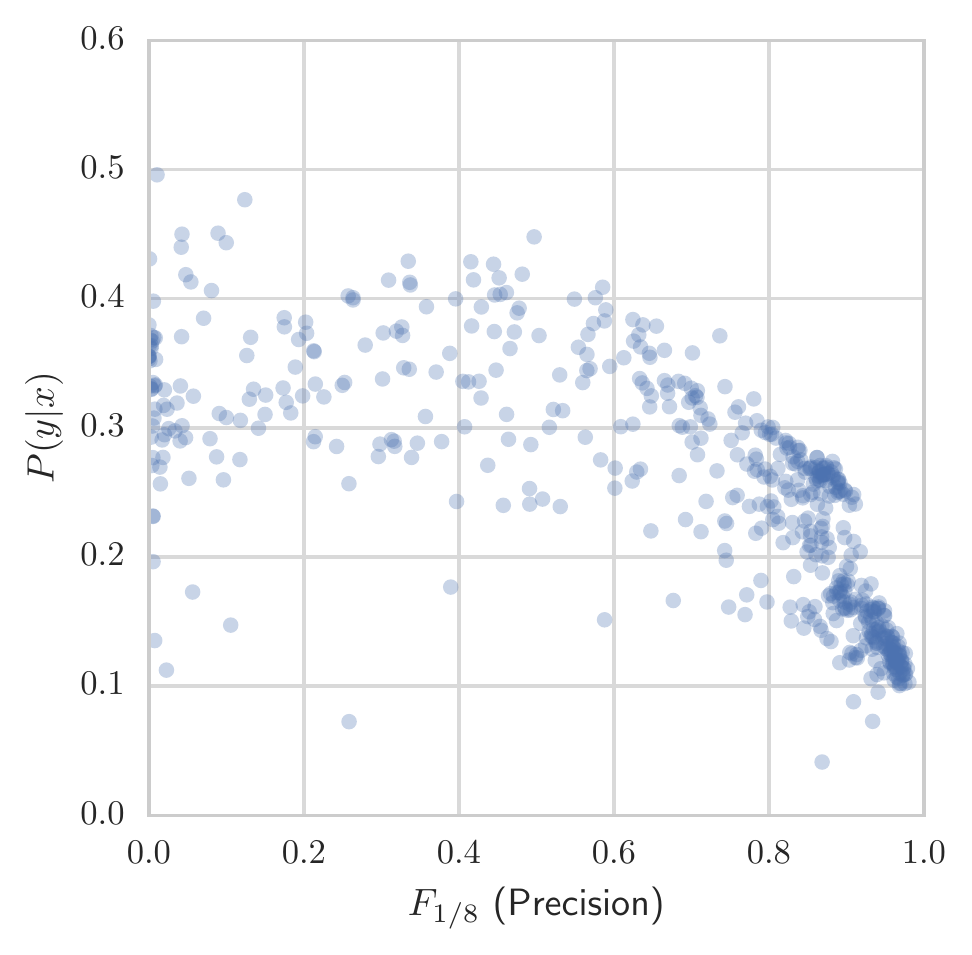}
    \includegraphics[width=0.45\linewidth]{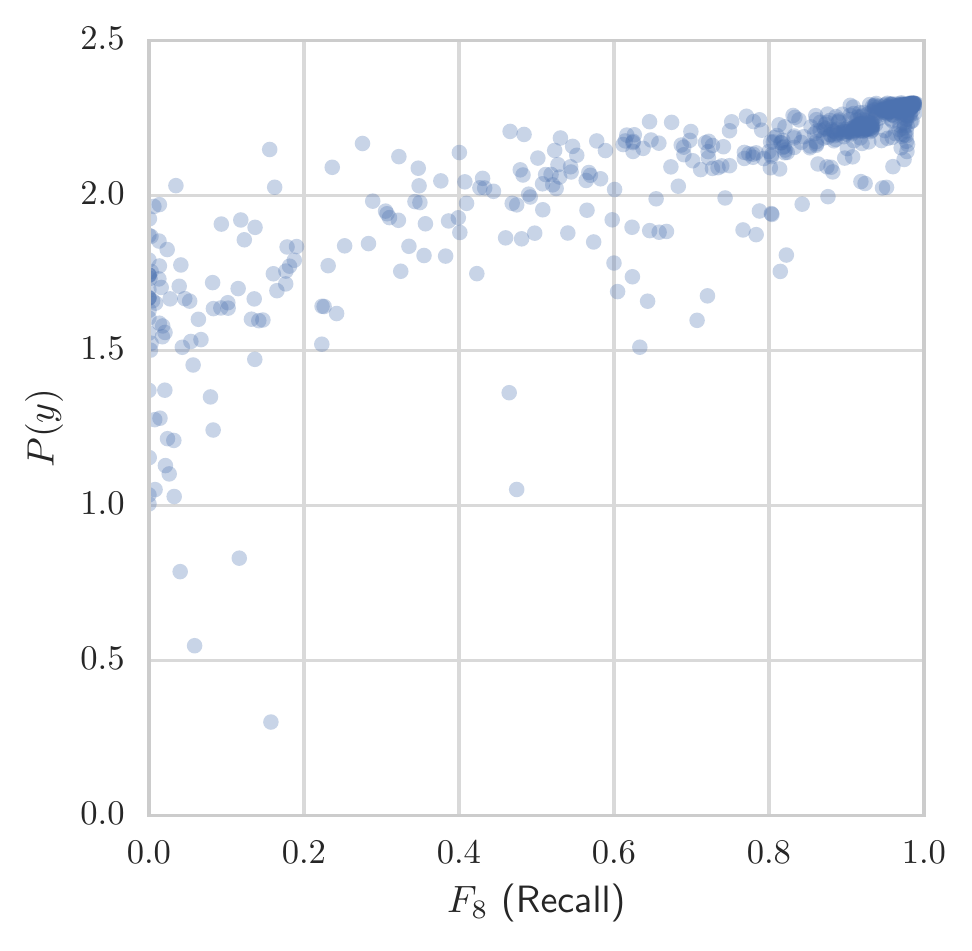}
    \caption{Comparing our unsupervised $F_{1/8}$ and $F_8$ measures with the
    supervised measures $P(y|x)$ and $P(y)$ similar to the IS (for a definition of $F_\beta$, see Section~\ref{sec:applicationtogans}).
    Each circle
    represents a trained generative model (GAN or VAE) on the MNIST data set.
    The values show a fairly high correlation with a Spearman rank correlation
    coefficient of -0.83 on the left and 0.89 on the right.}
    \label{fig:inception_score_scatter}
  \end{figure}

  \begin{figure}[t]
    \centering
    \includegraphics[width=0.378\linewidth]{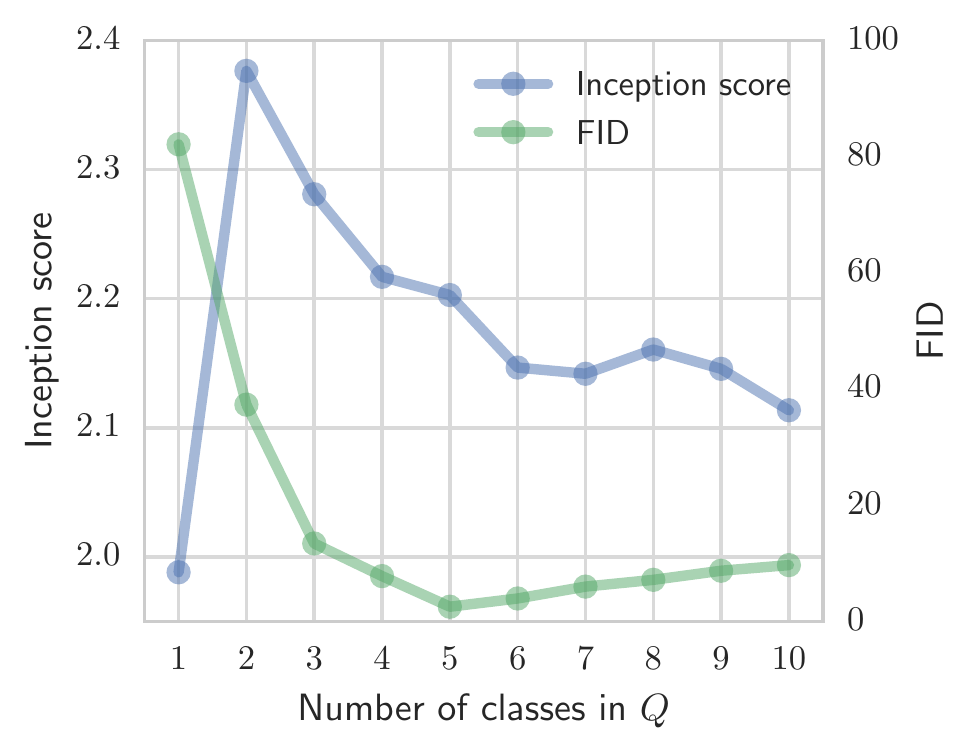}
    \includegraphics[width=0.3\linewidth]{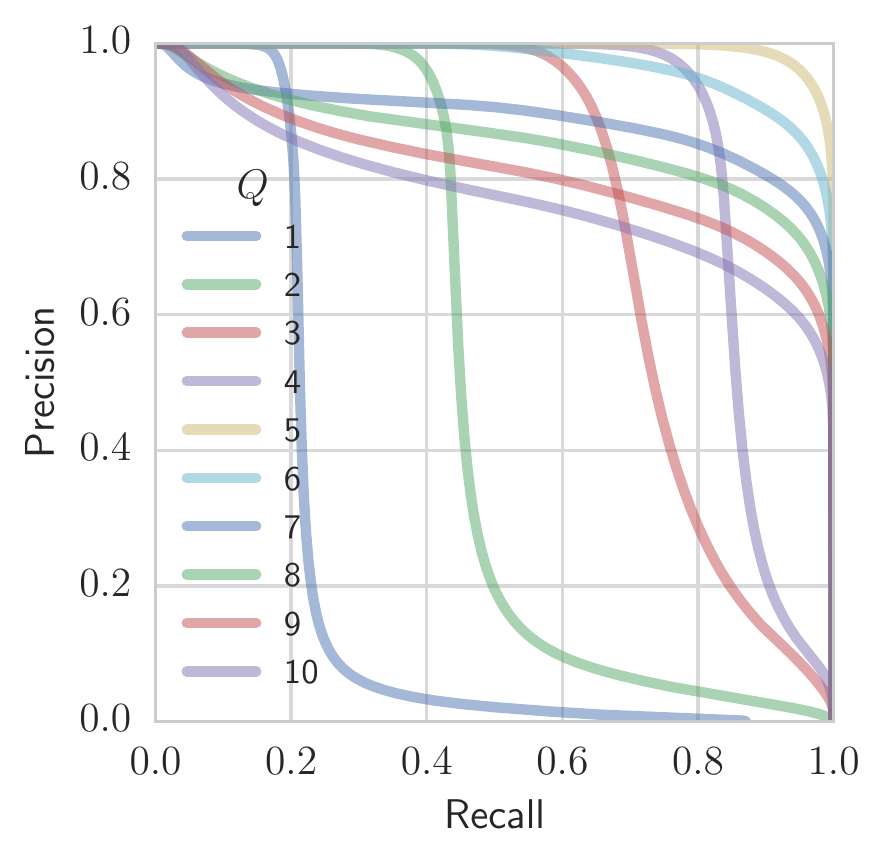}
    \includegraphics[width=0.378\linewidth]{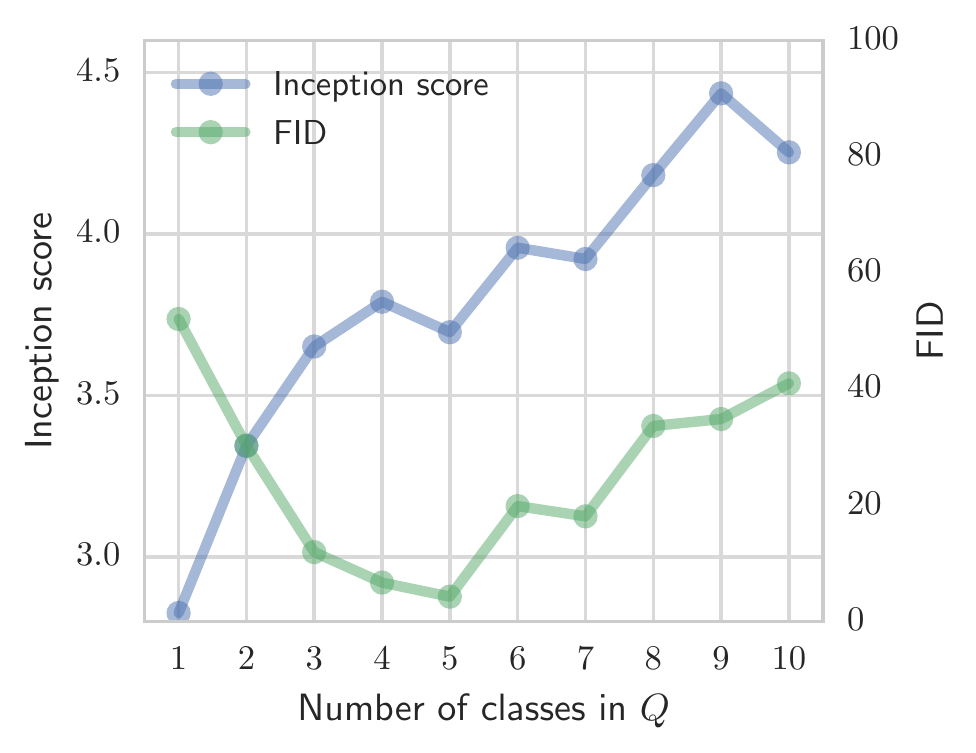}
    \includegraphics[width=0.3\linewidth]{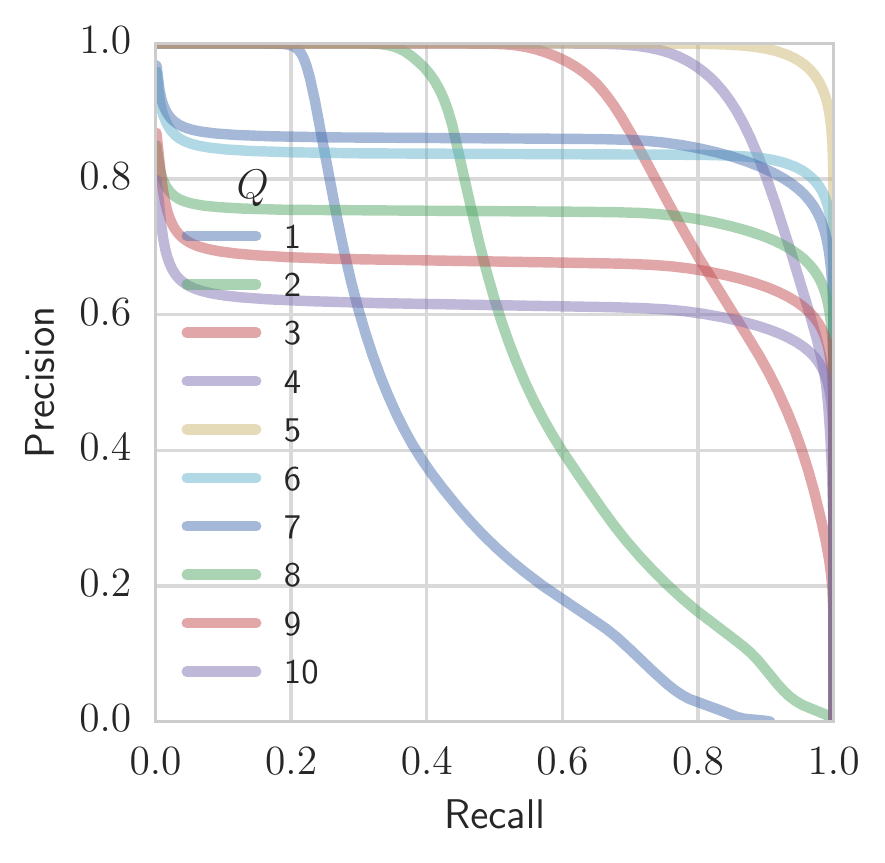}
    \caption{Corresponding plots as in Figure~\ref{fig:modes_exp} for the
    data sets MNIST (top) and Fashion-MNIST (bottom).}
    \label{fig:modes_exp_appendix}
  \end{figure}

  \begin{figure}[t]
    \begin{center}
      \setlength{\tabcolsep}{2pt}
      \begin{tabular}{cc}
        \begin{tabular}{cc}
          \includegraphics[width=0.485\linewidth]{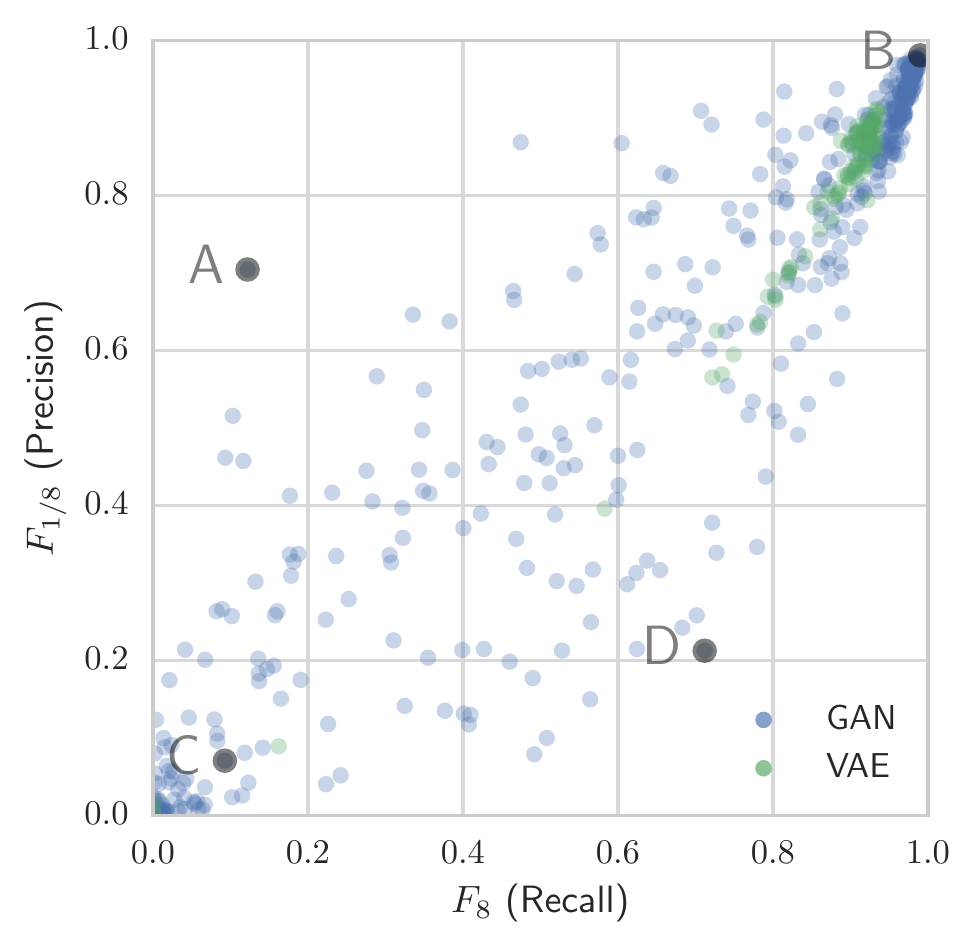}
        \end{tabular}
        \begin{tabular}{cc}
          \begin{tikzonimage}[width=0.21\textwidth]{samples_mnist_792.jpg}[image label]
              \node{A};
          \end{tikzonimage} &
          \begin{tikzonimage}[width=0.21\textwidth]{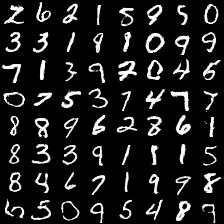}[image label]
              \node{B};
          \end{tikzonimage} \\
          \begin{tikzonimage}[width=0.21\textwidth]{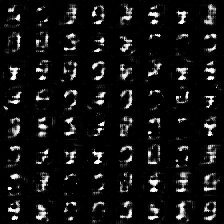}[image label]
              \node{C};
          \end{tikzonimage} &
          \begin{tikzonimage}[width=0.21\textwidth]{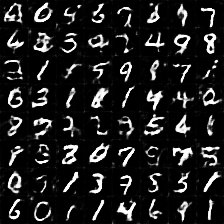}[image label]
              \node{D};
          \end{tikzonimage} \\
        \end{tabular}
      \end{tabular}
      \begin{tabular}{cc}
        \begin{tabular}{cc}
          \includegraphics[width=0.485\linewidth]{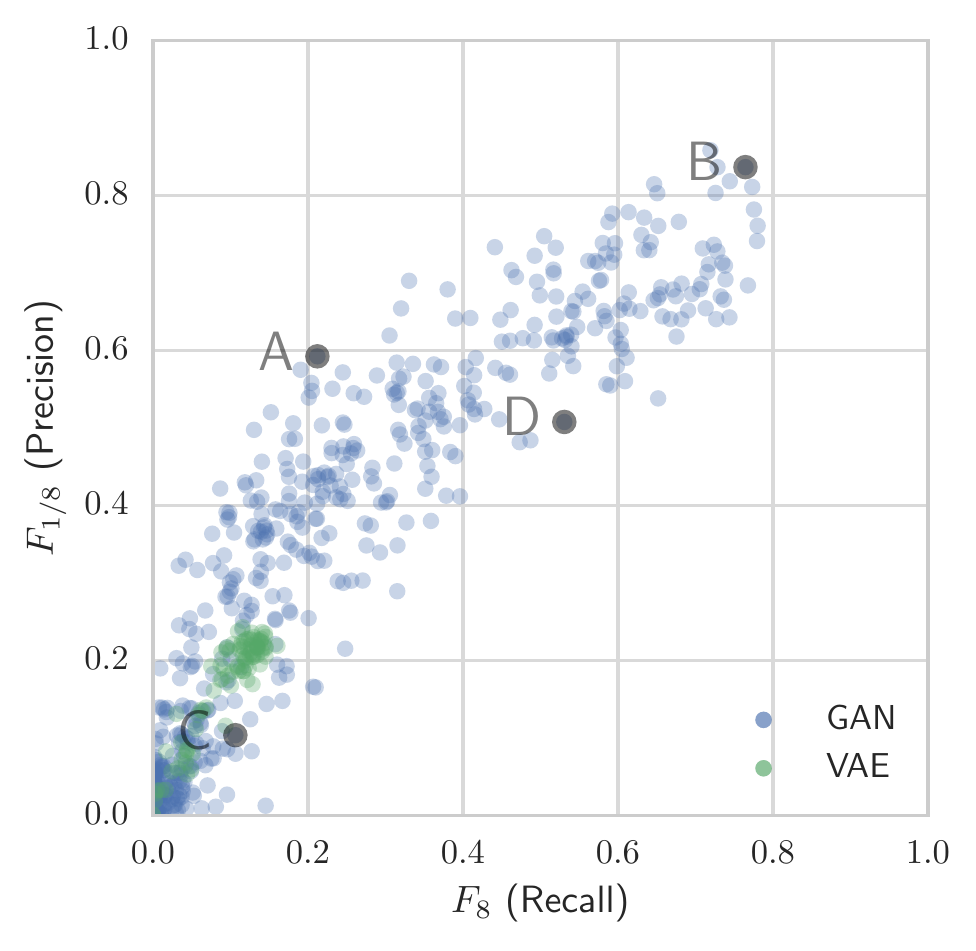}
        \end{tabular}
        \begin{tabular}{cc}
          \begin{tikzonimage}[width=0.21\textwidth]{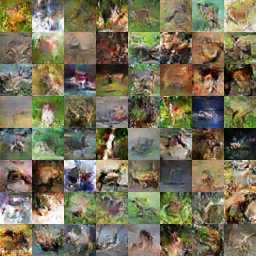}[image label]
              \node{A};
          \end{tikzonimage} &
          \begin{tikzonimage}[width=0.21\textwidth]{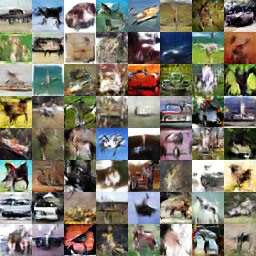}[image label]
              \node{B};
          \end{tikzonimage} \\
          \begin{tikzonimage}[width=0.21\textwidth]{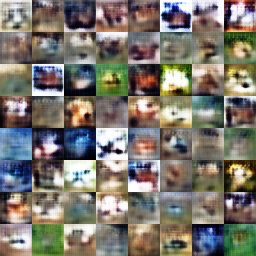}[image label]
              \node{C};
          \end{tikzonimage} &
          \begin{tikzonimage}[width=0.21\textwidth]{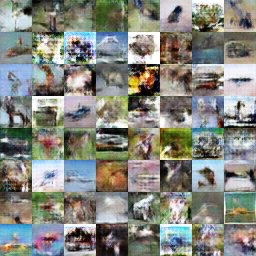}[image label]
              \node{D};
          \end{tikzonimage} \\
        \end{tabular}
      \end{tabular}
      \begin{tabular}{cc}
        \begin{tabular}{cc}
          \includegraphics[width=0.485\linewidth]{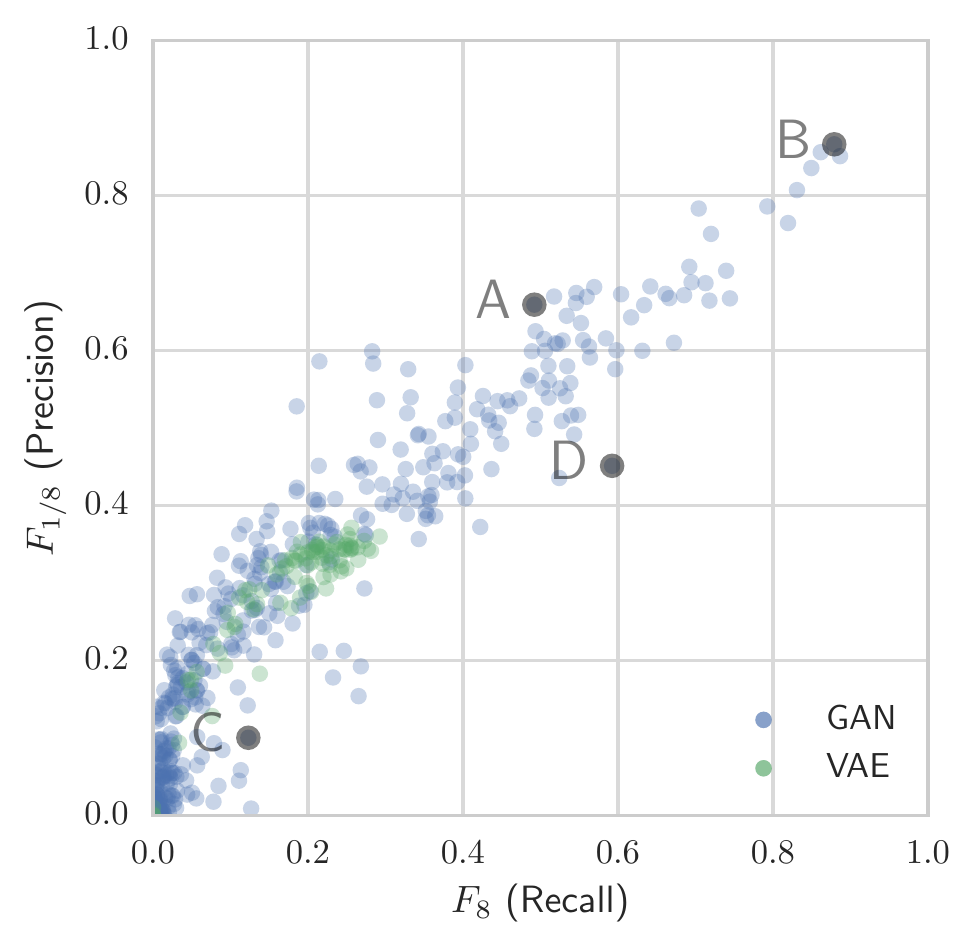}
        \end{tabular}
        \begin{tabular}{cc}
          \begin{tikzonimage}[width=0.21\textwidth]{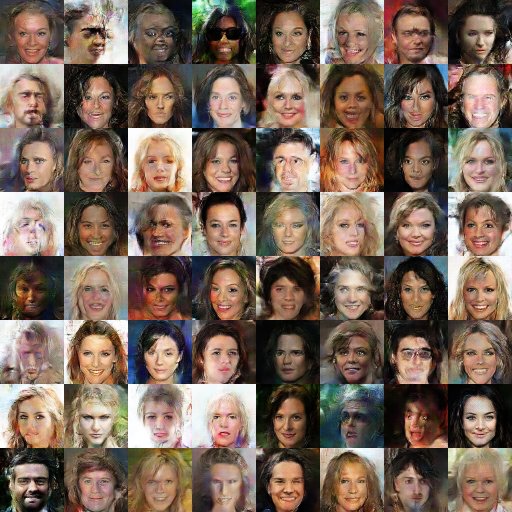}[image label]
              \node{A};
          \end{tikzonimage} &
          \begin{tikzonimage}[width=0.21\textwidth]{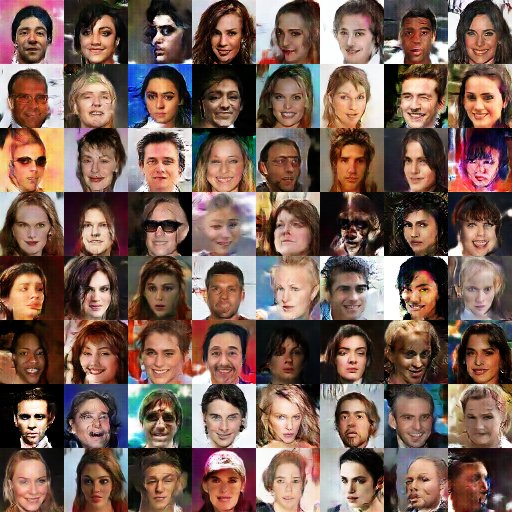}[image label]
              \node{B};
          \end{tikzonimage} \\
          \begin{tikzonimage}[width=0.21\textwidth]{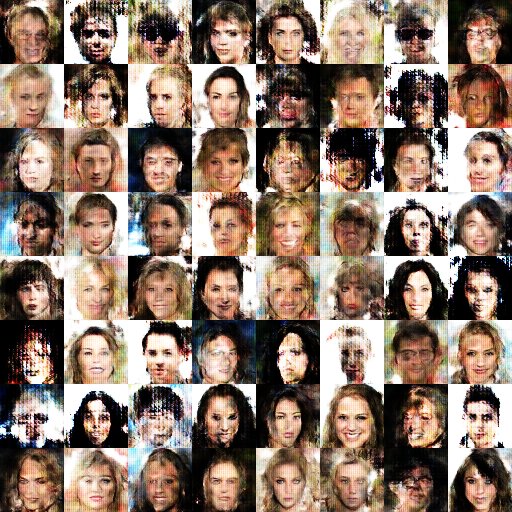}[image label]
              \node{C};
          \end{tikzonimage} &
          \begin{tikzonimage}[width=0.21\textwidth]{samples_celeba_305.jpg}[image label]
              \node{D};
          \end{tikzonimage} \\
        \end{tabular}
      \end{tabular}
    \end{center}
    \hfill
    \vspace{-3mm}
    \caption{Corresponding plots as in Figure~\ref{fig:all_gans_2d} for the
    data sets MNIST (top), CIFAR-10 (middle) and CelebA (bottom).}
    \label{fig:all_gans_2d_others}
  \end{figure}

\end{document}